\documentclass[11pt]{article}

\usepackage[utf8]{inputenc}  
\usepackage[T1]{fontenc}
\usepackage{lmodern}
\usepackage[english]{babel}

\usepackage{geometry}
\usepackage{amsmath, amssymb, amsfonts}
\usepackage{amsthm}           
\usepackage{bbm}
\usepackage{cancel}

\usepackage{xcolor}           
\usepackage{graphicx}
\usepackage{subfig}
\usepackage{wrapfig}
\usepackage{tikz-cd}
\usepackage{pdflscape}

\usepackage{booktabs}
\usepackage{array}
\usepackage{multirow}
\usepackage{makecell}

\usepackage[inline]{enumitem} 
\usepackage[normalem]{ulem}
\usepackage{verbatim}
\usepackage{epigraph}
\usepackage{changepage}
\usepackage{todonotes}

\usepackage[ruled,vlined]{algorithm2e}
\DontPrintSemicolon
\SetAlFnt{\small}          
\SetAlCapFnt{\small} 

\usepackage{authblk}

\usepackage{natbib}
\usepackage{url}

\usepackage{hyperref}
\hypersetup{
    colorlinks=true,
    citecolor=blue
}

\usepackage[capitalize,noabbrev]{cleveref}

\usepackage{microtype}                 
\newcommand{\reals}{\mathbb{R}}
\newcommand{\supp}{\textnormal{supp}}
\newcommand{\E}{\mathbb{E}}

\newcommand{\naturals}{\mathbb{N}}
\newcommand{\expec}[1]{\mathbb{E}\left[#1\right]}
\newcommand{\expecover}[2]{\mathbb{E}_{#1}\left[#2\right]}
\newcommand{\prob}[1]{\mathbb{P}\left[#1\right]}
\newcommand{\kl}{\mathrm{KL}}
\newcommand{\indicator}[1]{\mathbf{1}\left\{#1\right\}}

\newcommand{\indexel}{\mathrm{U}^{\mathrm{EL}}}
\newcommand{\logel}{\log \mathrm{EL}}
\newcommand{\crad}{C_{\mathrm{rad}}}
\newenvironment{proofsketch}{\begin{proof}[Proof sketch.] }{\end{proof}}
\newenvironment{proofof}[1]{\begin{proof}[Proof of {#1}]}{\end{proof}}

\SetKwComment{Cmt}{$\triangleright$\ }{}
\newcommand{\Comment}[1]{\Cmt*[r]{#1}}
\SetKwInput{KwInit}{Initialization}

\newtheorem{theorem}{Theorem}[section]
\newtheorem{definition}[theorem]{Definition}

\newtheorem{proposition}[theorem]{Proposition}

\newtheorem{corollary}[theorem]{Corollary}
\newtheorem{lemma}[theorem]{Lemma}
\theoremstyle{definition}\newtheorem{remark}[theorem]{Remark}

\title{Reliable Scale-Free Bandits via Empirical Likelihood}

\author{Pratham Gupta\thanks{Equal contribution.}}
\author{Alhad Sethi\protect\footnotemark[1]}
\author{Shubhada Agrawal}
\affil{Indian Institute of Science\\
  \small\texttt{\{%
    \href{mailto:prathamgupta@iisc.ac.in}{prathamgupta},
    \href{mailto:alhadsethi@iisc.ac.in}{alhadsethi},
    \href{mailto:shubhada@iisc.ac.in}{shubhada}%
  \}@iisc.ac.in}}

\date{\today}

\begin{document}
\maketitle
\begin{abstract}
        We study stochastic bandits without knowledge of reward scales or tail parameters. Almost-sure regret guarantees describe eventual learning but do not quantify expected performance or the risk of large regret in a finite run. We introduce \texttt{EL-F}, an anytime empirical-likelihood (EL) extension of the \(g\)-Forcing policy of \citet{Cowan_Katehakis_2020}, and characterize its performance at three levels. First, for integrable rewards and non-decreasing, diverging exploration schedules $g$ with vanishing increments, \texttt{EL-F} preserves the $g$-Forcing asymptotics for pseudo-regret $R(T)$: \(R(T)/g(T)\to\sum_{i:\Delta_i>0}\Delta_i\) almost surely as $T\to\infty$, where $\Delta_i$ is the suboptimality gap of arm $i$. Second, on every sub-exponential instance, for non-decreasing \(g(t)=\omega(\log t)\) satisfying a mild feasibility condition, expected pseudo-regret is at most \(g(T)\sum_{i:\Delta_i>0}\Delta_i+C_{\underline{P}}\), where \(C_{\underline{P}}\) is an instance-dependent constant independent of $T$. Moreover, we also establish finite-time tail bounds with super-polynomial decay in regret beyond the forced-exploration budget. For slowly varying \(g\), we determine the exact exponent of linear regret deviations on two-armed Gaussian instances with common variance and establish a matching lower bound among policies satisfying the same expected-regret budget on every such Gaussian instance. These quantitative guarantees also cover the greedy variant obtained at zero EL radius. A fixed positive radius introduces local optimism while preserving the leading asymptotic regret. We establish coupled pathwise dominance over the greedy variant in two-armed bandits with a deterministic suboptimal arm. Experiments demonstrate reductions in prolonged suboptimal allocation compared with \(g\)-Forcing on challenging instances.
\end{abstract}

\section{Introduction}
The stochastic multi-armed bandit is a classic statistical model for sequential decision making. At each time, the decision maker (or algorithm) chooses one of the $K$ arms and observes a reward. The algorithm's goal is to maximize total expected reward, equivalently minimize the expected \emph{regret}. This framework captures the quintessential trade-off between exploration and exploitation and has been studied extensively; see \citet{Slivkins2019, Lattimore2020}.
 
The regret achievable in a stochastic bandit problem depends on the class of reward distributions the learner must accommodate. To make this dependence precise, let \(\mathcal L\) denote the ambient class of admissible reward distributions, and write \(m(P)\) for the mean of a distribution \(P\). A bandit instance with \(K\) arms is a vector \(\underline{P}=(P_1,\ldots,P_K)\in\mathcal L^K\), where \(P_i\) is the reward distribution of arm \(i\). Let $A_t$ and $X_t$ denote the arm chosen and the reward observed at time $t$. The pseudo-regret up to time $T$ is $
R(T) = \sum_{i: \Delta_i >0} \Delta_i N_i(T+1),
$ where, for $m^*=\max_j m(P_j)$, $\Delta_i = m^* - m(P_i)$ is the sub-optimality gap and $N_i(t)$ is the number of pulls of arm $i$ before time $t$. Pseudo-regret isolates the loss due to arm selection from reward noise. We study its almost-sure    asymptotics, expectation, and finite-time tails, which describe complementary aspects of an algorithm’s performance. 
 
To achieve low regret, the learner must distinguish a suboptimal arm's reward distribution from alternatives under which that arm would be optimal. The difficulty of doing so depends on how close these alternatives can be to the original distribution, as measured by $\kl$ divergence, within the ambient class $\mathcal L$. For a distribution $P$ and a candidate mean $m\in\reals$, define the information projection
$\kl_{\inf}^{\mathcal{L}}(P, m) := \inf\left\{\kl(P,Q):Q \in \mathcal{L}, m(Q) \geq m\right\}.$ \citet{LaiRobbins1985,burnetas1996optimal} showed that any consistent algorithm\footnote{An algorithm is consistent on $\mathcal{L}$ if for every instance in $\mathcal{L}^K$ and every suboptimal arm $i$, $\expec{N_i(T)}= o(T^{\alpha})$ for all $\alpha\in (0,1)$.} satisfies
$$
\liminf_{T\to \infty} \frac{\expec{R(T)}}{\log T} \geq \sum_{i: \Delta_i > 0}\frac{\Delta_i}{\kl_{\inf}^{\mathcal{L}}(P_i, m^*)}.
$$
For suitably restricted ambient classes, this lower bound is attained by algorithms such as Thompson Sampling \citep{Thompson1933, KaufmannKordaMunos2012, AgrawalGoyal2017}, $\kl$-UCB \citep{CappeEtAl2013, agrawal2021regret} and IMED \citep{honda2015non}. 

However, if $\mathcal{L}$ is too rich, then $\kl_{\inf}^{\mathcal{L}}(P_i, m^*) = 0$ and logarithmic expected regret is impossible. For example, consistency on all bounded-support distributions $\mathcal{P}_{\textnormal{bdd}}=\bigcup_{a<b}\mathcal{P}_{[a,b]}$, where $\mathcal{P}_{[a,b]}$ is the set of distributions supported on $[a,b]$, already forces $\kl_{\inf}^{\mathcal{L}}=0$ for $\mathcal L = \mathcal P_{\textnormal{bdd}}$ \citep[Theorem 11]{hadiji2023adaptation}; see also \citet{lattimore2017scale, pmlr-v117-agrawal20a}. To recover logarithmic expected regret, one must therefore impose additional restrictions on $\mathcal L$, for instance via a known bound on the kurtosis \citep{lattimore2017scale}, moment conditions \citep{bubeck2013bandits,agrawal2021regret}, quantile conditions \citep{baudry2021optimality}, or a mass condition on the negative half-line \citep{genalti2024varepsilon}. 

An alternative to imposing additional distributional restrictions is to relax the target of logarithmic expected regret. Our finite-time expected-regret and tail guarantees concern the class of all light-tailed (sub-exponential) distributions, without known reward scales or tail parameters. Consistency over this class rules out logarithmic expected regret, but existing algorithms can approach this rate arbitrarily closely: expected regret $\mathcal{O}(g(T))$ is attainable with $g(t)/\log t\to\infty$ arbitrarily slowly, without knowledge of the reward scales or tail parameters \citep{ashutosh2021bandit}. We refer to such algorithms as \emph{scale-free}. \texttt{RDS} is another scale-free method with expected-regret guarantees on this class \citep{baudry2021optimality}. Related scale-free approaches have also been studied for bounded and sub-Gaussian families \citep{hadiji2023adaptation,han2024forced}.

A complementary perspective is provided by \citet{Cowan_Katehakis_2020}, who study the \(g\)-Forcing policy. It samples a least-sampled arm at time $t$ whenever some arm has fewer than \(g(t)\) observations, and otherwise selects an arm with the largest empirical mean. For integrable rewards and suitable slowly increasing exploration schedules $g$, its pseudo-regret satisfies ${R(T)}/{g(T)} \longrightarrow \sum\nolimits_{i:\Delta_i>0}\Delta_i$ almost surely. However, these pathwise asymptotics do not provide finite-time bounds on expected regret or control the random time required to reach the asymptotic regime, which may be large. This motivates a quantitative analysis of the forcing architecture.

Regret tails quantify the probability of large regret in a single run, a concern in safety-critical applications such as clinical trials. These tails have been studied extensively under well-specified distributional assumptions \citep{audibert2009exploration,Salomon2011,fan2025fragility,simchi2023regret,Panda2026}. In particular, \citet{fan2025fragility} show that optimized algorithms can have heavy regret tails and that model misspecification can substantially worsen their performance. They propose increasing exploration to improve the polynomial tail decay rate, at the cost of larger expected regret. While this reduces fragility within the specified ambient class, it is still not scale-free. 

A closely related result is that of \citet{simchi2023regret}, whose anytime UCB policy provides regret-tail guarantees for sub-exponential noise. Their analysis assumes that the arm means lie in a known range, and the sub-exponential extension introduces an additional tuning parameter. Existing scale-free algorithms such as \texttt{R-UCB} and  \texttt{RDS} provide expected-regret guarantees but do not characterize the tails of pseudo-regret. See \Cref{tab:comparison} for a snapshot summary. We therefore ask:
\begin{quote}
Can scale-free forced exploration combine almost-sure guarantees, near-logarithmic expected regret, and strong finite-time tail bounds?
\end{quote}

We answer this question with \texttt{EL-F}, an anytime empirical-likelihood (EL) extension of the \(g\)-Forcing architecture. \texttt{EL-F} retains forced exploration and uses a fixed-radius EL index for decisions between forced pulls. At zero radius, the index equals the empirical mean, giving the greedy forcing variant. Our finite-time expected-regret and tail guarantees hold for every fixed nonnegative radius, including zero. Under finite absolute means, we also show that the EL index is strongly consistent, allowing \texttt{EL-F} to preserve the almost-sure asymptotics of $g$-Forcing.

Forced exploration ensures continued sampling of arms that currently appear inferior and prevents the permanent trapping illustrated in \Cref{sec:false_start}. A positive EL radius introduces optimism adapted to the observed rewards. This local optimism can encourage an arm to be revisited before its next forced pull. We establish a coupled pathwise comparison with the greedy variant in two-armed bandits with a deterministic suboptimal arm, and experiments show reductions in prolonged suboptimal allocation on broader challenging instances.


\paragraph{Contributions.} $\texttt{EL-F}$ requires neither a known range for the arm means nor distribution-specific parameters. We analyze \texttt{EL-F} for every fixed EL radius $C_{\mathrm{rad}} \ge 0$. Let $\Delta := \sum\nolimits_{i}\Delta_i$. 
\begin{itemize}
    \item \textit{Almost-sure regret.} For integrable rewards and positive, non-decreasing exploration schedules satisfying \(g(t)\to\infty\) and \(g(t+1)-g(t)\to0\), we prove $R(T)/g(T)\longrightarrow \Delta$  almost surely. This extends the \(g\)-Forcing asymptotics to \texttt{EL-F} with fixed EL radius (Theorem~\ref{thm:expected_regret}(1)).

    \item \textit{Expected regret.} On any sub-exponential instance, for positive, non-decreasing schedules satisfying \(g(t)=\omega(\log t)\) and \(g(t)\le t-1\) for \(t\ge2\),  \texttt{EL-F} satisfies $\expec{R(T)} \le g(T) \Delta + C_{\underline{P}}$, where $C_{\underline{P}}$ is finite and independent of $T$ (\Cref{thm:expected_regret}(2)).

    \item \emph{Regret tails.} Under the same assumptions as the expected-regret result, for every $T, x\ge1$,
    \[\log \prob{R(T)\ge (x+\lceil g(T)\rceil)\Delta}\le  \log((K-1)C_0\lceil x\rceil) {-C_{\min}g(x)},\]
    where $C_0,C_{\min}>0$ are instance-dependent constants (\Cref{thm:tail}). This bound decays faster than any inverse polynomial in $x$, which measures regret beyond the forced exploration. 
    
    \item \textit{Exact Gaussian exponent and optimality.} If $g$ is also slowly varying, then on every two-armed Gaussian instance with common variance $\sigma^2$ and gap $\Delta>0$, for each $c\in (0,\Delta)$, $\log\prob{R(T)>cT} = -(1+o(1))\frac{\Delta^2}{2\sigma^2}\,g(T)$. A matching change-of-measure lower bound shows that \texttt{EL-F} achieves the best possible decay exponent for these linear-regret tails: no policy satisfying the same instance-wise expected-regret bound on two-armed Gaussian instances can improve this exponent (Proposition~\ref{thm:simchi_lb} and Corollary~\ref{cor:gaussian}).

    \item \textit{Benefit of local optimism and experiments.} We show that  \texttt{EL-F} incurs no greater pseudo-regret than $g$-Forcing in two-armed bandits with a deterministic suboptimal arm (Lemma~\ref{lem:Crad-better}). We compare \texttt{EL-F} with existing scale-free algorithms on light-tailed instances and demonstrate its competitiveness across instances. The EL radius ablation demonstrates reductions in prolonged suboptimal allocation on challenging instances (\Cref{sec:experiments} and Appendix~\ref{app:experiments}).
\end{itemize}

\begin{table}[h]
\centering
\footnotesize
\renewcommand{\arraystretch}{1.15}
\setlength{\tabcolsep}{3pt}
\newcommand{\yes}{\ensuremath{\surd}}
\newcommand{\no}{\ensuremath{\times}}
\newcommand{\as}{~\textnormal{a.s.}}
\newcommand{\inexp}{~\textnormal{in exp.}}
\begin{tabular}{@{}>{\raggedright\arraybackslash}p{0.19\linewidth}>{\centering\arraybackslash}p{0.08\linewidth}>{\centering\arraybackslash}p{0.09\linewidth}>{\raggedright\arraybackslash}p{0.13\linewidth}>{\raggedright\arraybackslash}p{0.21\linewidth}>{\raggedright\arraybackslash}p{0.18\linewidth}@{}}
\toprule
Algorithm & Anytime & Scale-free & Reward class & Regret guarantee & Regret tail \\
\midrule
$g$-Forcing \citep{Cowan_Katehakis_2020} & \yes & \yes & finite mean\textsuperscript{a} & $R(T)/g(T)\to\sum_i\Delta_i$\as& not analyzed \\
\texttt{R-UCB} \citep{ashutosh2021bandit} & \yes & \yes & sub-exp. & $\mathcal{O}(g(T))$\inexp & not analyzed \\
\texttt{RDS} \citep{baudry2021optimality} & \yes & \yes & sub-exp. & $\mathcal{O}(\log T\log\log T)$\inexp \textsuperscript{b} & not analyzed \\
\texttt{FE} \citep{han2024forced} & \no\rlap{\textsuperscript{c}} & \yes & sub-Gaussian & $\mathcal{O}(\log^2 T)$\inexp\textsuperscript{d} & not analyzed \\
\texttt{AHB} \citep{hadiji2023adaptation} & \yes & \yes & bounded\textsuperscript{e} &  $\mathcal{O}(T^{1-\alpha})$, $\alpha\in[\frac12,1)$\inexp & not analyzed \\
\midrule
\citet{simchi2023regret} & \yes & \no\rlap{\textsuperscript{f}} & sub-Gaussian, sub-exp. & $\mathcal{O}(f(T))$\inexp  & $e^{-\Theta(f(x))}$; matching l.b. \\
\citet{fan2025fragility} & \yes & \no\rlap{\textsuperscript{g}} & exp.\ family &  $\mathcal{O}(f(T))$\inexp & $e^{-\Theta(f(x))}$; matching l.b. \\
\midrule
\multirow{2}{*}{\textbf{\texttt{EL-F} (ours)}} & \multirow{2}{*}{\yes} & \multirow{2}{*}{\yes} & integrable & $R(T)/g(T)\to\sum_i\Delta_i$\as  & --- \\
 & & & sub-exp. & $g(T)\sum_i \Delta_i + \mathcal{O}(1)$\inexp, any $g=\omega(\log t)$ & $e^{-\Omega(g(x))}$; exact exponent on Gaussians \\
\bottomrule
\end{tabular}
\caption{Comparison with existing work; $f\ge\log t$, $g=\omega(\log t)$, in exp.: in expectation, a.s.: almost surely. l.b.: lower bound \textsuperscript{a}finite variance assumed for remainder-term bounds. \textsuperscript{b}For $\rho_n=\mathcal{O}(\log n)$. \textsuperscript{c}Near-logarithmic schedules take $T$ as input; anytime schedule $f(r)=r$ gives $\sqrt{T}$ regret. \textsuperscript{d}Equation 12 in their work. \textsuperscript{e}Unknown range. Distribution-free regret is $\mathcal{O}(T^{\alpha})$; show lower bound that distribution-free regret $\Phi(T)=o(T)$ implies distribution-dependent regret $\Omega(T/\Phi(T))$. \textsuperscript{f}Means assumed in $[0,1]$. \textsuperscript{g}$\kl$-UCB for known one-parameter exponential family, $f(t)=(1+b)\log t$ gives tails $x^{-(1+b)}$. 
\label{tab:comparison}}  
\end{table}

\noindent\textbf{Organization.} \Cref{tab:comparison} summarizes the comparison with existing methods, and \Cref{sec:related} discusses the related work. \Cref{sec:preliminaries} introduces notation and background. The failure example in  \Cref{sec:false_start} motivates the design of the algorithm in \Cref{sec:algorithm}. We analyze the almost sure and expected regrets in \Cref{sec:analysis} and the regret tails in \Cref{,sec:tails}, and report experiments in \Cref{sec:experiments}. Detailed experimental evaluation is presented in Appendix~\ref{app:experiments}. 

\section{Preliminaries}\label{sec:preliminaries}
We first fix notation and review the bandit model and empirical likelihood. Let $\reals$ and $\naturals$ denote the sets of real and natural numbers, respectively. For $k \in \naturals$, write $[k] := \{1, 2, \dots, k\}$. We denote the set of probability distributions on $\reals$ by $\mathcal{P}(\reals)$, and the Dirac mass at $x$ by $\delta_x$. We adopt the convention $\inf \emptyset = +\infty$. For discrete $P\in \mathcal{P}(\reals)$, we let $\supp(P) := \{x\in \reals: P(x) >0\}$. 

\paragraph{Bandit conventions. } Probabilities and expectations are taken under the joint law induced by the algorithm and the bandit instance $\underline{P}$; see \citet[Chapter 4]{Lattimore2020} for the canonical construction. We represent the rewards of each arm $i\in[K]$ by an i.i.d.\ sequence $(X_{i,s})_{s\ge1}$ with law $P_i$, with these sequences independent across arms and independent of the algorithm's internal randomization. The $s^\text{th}$ pull of arm $i$ returns $X_{i,s}$. We assume there is at least one sub-optimal arm and label the optimal arm as arm $1$, without loss of generality, so that $\Delta_i=m(P_1)-m(P_i)$. For $s\ge1$, define the empirical distribution and empirical mean of the first $s$ samples of arm $i$ by $\hat{P}_{i,s} = \frac{1}{s}\sum_{j=1}^{s}\delta_{X_{i,j}}$ and $\hat{m}_{i,s} = \frac{1}{s}\sum_{j=1}^{s}X_{i,j}$, respectively. Recall that $N_i(t)$ counts the pulls of arm $i$ before time $t$. Whenever $N_i(t)\ge1$, we write $\hat{P}_i(t)=\hat{P}_{i,N_i(t)}$ and $\hat{m}_i(t)=\hat{m}_{i,N_i(t)}$. 

\paragraph{Ambient class under consideration. } We call a distribution \emph{light-tailed} if its moment generating function is finite in a neighbourhood of the origin. This is equivalent to sub-exponentiality \citep[Theorem 2.13]{Wainwright2019}. Let $\mathcal{P}_{\exp}=\bigcup_{\nu,\alpha>0}\mathcal{P}_{\nu,\alpha}$ for the class of all sub-exponential distributions, where for $\nu,\alpha>0$, 
$\mathcal{P}_{\nu,\alpha}:=\{Q\in\mathcal{P}(\reals):\ \mathbb{E}_{Q}[e^{\lambda(X-m(Q))}]\le e^{\lambda^2\nu^2/2},\ \ \forall|\lambda|\le 1/\alpha\}.$ We note that while \texttt{EL-F} does not need any distributional knowledge whatsoever, the {expected regret} and regret-tail guarantees that we prove, hold for instances in $\mathcal{P}^K_{\exp}$, and the almost sure guarantees only require the existence of a finite first moment. 

\paragraph{Empirical Likelihood.}
Our algorithm \texttt{EL-F} combines forced exploration with an upper confidence bound (UCB) phase. The index it uses in the UCB phase is constructed using empirical likelihood (EL), a nonparametric inference method that provides asymptotic confidence intervals under minimal distributional assumptions. We briefly review this method here, and refer the reader to \citet{Owen2001} for a detailed exposition.

Let $X_1,\dots,X_n\in\reals$ be i.i.d.\ observations, and let their empirical distribution and empirical cumulative distribution function (CDF) be denoted by 
$\hat P_n=\tfrac1n\sum_{s=1}^n\delta_{X_s}$, and  $\hat F_n(x)=\tfrac1n\sum_{s=1}^n\indicator{X_s\le x}$, respectively. Then, for $Q\in\mathcal P(\reals)$ with CDF $F_Q$, the nonparametric likelihood of the samples is given by 
$L(Q)=\prod_{s=1}^n Q(\{X_s\})=\prod_{s=1}^n(F_Q(X_s)-F_Q(X_s^-)),$
where $F_Q(x^-)=Q((-\infty,x))$. Here, no parametric family is assumed. It is classical that $L(\cdot)$ is maximized by the empirical distribution, that is, $L(Q)\le L(\hat P_n)$ for every $Q$ \citep{Kiefer1956}; see also \citet[Chapter~2]{Owen2001}. 

Let $p$ be a distribution with weights $p_s> 0$ on $X_s$ so that $\sum_{s=1}^n p_s=1$. For simplicity of presentation, we assume no ties in the observations. Then, the nonparametric likelihood ratio of $p$ is then given by ${L(p)}/{L(\hat P_n)}=\prod_{s=1}^n np_s$. In analogy with parametric likelihood-ratio tests, we maximize this ratio over the weights whose mean is $m$, which gives the following profile (nonparametric) likelihood ratio
\begin{equation}
R(\hat{P}_n, m) = \sup \left\{ \prod_{s=1}^n n p_s \;:\; p_s > 0, \, \sum_{s=1}^n p_s = 1, \, \sum_{s=1}^n p_s X_s = m \right\}.
\end{equation}
Clearly, $R(\hat{P}_n, m) \in [0,1]$. In particular, $R(\hat P_n,m)=1$ when $m$ is the empirical mean, and it decreases as $m$ forces the weights further from uniform. If no feasible $p$ exists, i.e., $m\notin(\min_s X_s,\max_s X_s)$ for non-degenerate samples, we define $R(\hat P_n,m)=0$, since $R$ is a likelihood ratio and no reweighting of the data can attain such an $m$. Note that ties among the observations do not affect the profile likelihood ratio. For any fixed allocation of mass to the distinct values, $\prod_s np_s$ is maximized by splitting each value's mass equally among its copies, where it coincides with the likelihood ratio. See \citet[Section~2.3]{Owen2001} for details.


\begin{definition}[Profile Log-Likelihood Statistic]\label[definition]{def:logel}
  The profile log-likelihood statistic is given by
  \begin{equation*}
    \logel(\hat{P}_n, m) := -\tfrac{1}{n}\log R(\hat{P}_n, m) \in [0, +\infty],
  \end{equation*}
  with $-\log 0 = +\infty$. When $\min_s X_s < \max_s X_s$, $\logel(\hat{P}_n, m)$ is finite if and only if $m \in (\min_s X_s, \max_s X_s)$. When all samples are equal ($X_1 = \dots = X_n = c$), $\logel(\hat{P}_n, c) = 0$ and $\logel(\hat{P}_n, m) = +\infty$ for $m \neq c$.
\end{definition}
We note the following infimum-over-divergences formulation for profile log-likelihood statistic:
\begin{equation}\label{eq:infimum_logel_defn}
    \logel(\hat{P}_n, m) = \inf\{ \kl(\hat{P}_n, P) \;:\; P \in \mathcal{P}(\mathbb{R}), \, \supp(P) = \supp(\hat{P}_n), \, \mathbb{E}_P[X] = m \}.
\end{equation}
To see this, observe that for weights $(p_s)_{s}$ on the sample points, $\kl(\hat{P}_n, P) = \sum_s \frac1n\log(1/np_s)$, which equals $-\frac1n\log\prod_s np_s$. Minimizing over $p$ subject to $\sum_s p_sX_s=m$ gives $-\frac1n\log R(\hat{P}_n,m)$. 
  
Next, in the non-degenerate case, for any candidate mean $m\in (\min_s X_s, \max_s X_s)$, the profile log-likelihood statistic also admits the following dual  (cf. \citet[Section 2.9]{Owen2001}):
\begin{equation}\label{eq:dual_logel_defn}
    \logel(\hat{P}_n, m) = \sup_{\lambda \in D_m} \frac{1}{n} \sum_{s=1}^n \log \left( 1 + \lambda (m - X_s) \right),
  \end{equation}
where $D_m = \{\lambda \in \mathbb{R} : 1 + \lambda (m - X_s) > 0 \text{ for all } s = 1, \dots, n\}$.

Our focus in this work will be on adapting the EL method to the scale-free bandit setting and on exhibiting the finite-sample properties of our algorithm. To do so, we prove the following lower bound on the profile log-likelihood statistic that may be of independent interest. 
\begin{lemma}\label[lemma]{lem:logel_lower_bound}
Let $\hat P_n$ be the empirical distribution of $X_1,\dots,X_n\in\reals$, with mean $m(\hat P_n)$, variance $\sigma^2(\hat P_n)$ and largest element $X_{\max}>m(\hat P_n)$. Then for every $m\in\bigl(m(\hat P_n),X_{\max}\bigr)$,
\begin{equation}\label{eq:logel_lb_on_feasiblemeans}
\logel(\hat{P}_n, m) \geq \frac{0.5(m - m(\hat P_n))^2}{ \sigma^2(\hat P_n) + (m - m(\hat P_n))(X_{\max} - m(\hat P_n)) }.
\end{equation}
\end{lemma}

A proof of \Cref{lem:logel_lower_bound} can be found in \Cref{app:main_text_lemma_proofs}. A key step in the proof is expressing $\logel(\hat{P}_n, m)$ using its dual representation (\Cref{eq:dual_logel_defn}) and applying the following logarithmic inequality $\log(1+x) \geq x-0.5{x^2}/(1-c)$ that holds for any $c\in (0,1), x\geq -c$. We refer the reader to \Cref{sec:extra_lemmas} for a proof of this inequality. We note that variants of this algebraic inequality are used to show (empirical) Bernstein-type bounds in the literature \citep{audibert2009exploration, maurer2009empirical,blm2013concentration}. 
\paragraph{Relating $\logel(\cdot,\cdot)$ and $\kl_{\inf}^{\mathcal{L}}(\cdot,\cdot)$.}
The non-parametric $\kl$-UCB index of \citet{agrawal2021regret} is built using the statistic given as $\kl_{\inf}^{\mathcal{L}}(\hat P_n,m)=\inf\{\kl(\hat P_n,Q): Q\in\mathcal{L},\, m(Q)\ge m\}$, which can be understood as the closest distribution in the given class $\mathcal L$ that has mean at least $m$. Its definition, therefore, requires fixing $\mathcal{L}$ in advance. Moreover, for rich $\mathcal L$, it is identically zero. The comparison with $\logel$ is easiest to see from \Cref{eq:infimum_logel_defn}, which is a similar information projection, with the ambient class $\mathcal{L}$ replaced by the \emph{data-dependent} class of distributions supported on the observed samples. 

Two consequences drive the rest of the paper. First, $\logel$ needs no specification of $\mathcal{L}$, which is what makes an index built from it scale-free. Second, under $\kl_{\inf}^{\mathcal{L}}$, a rich $\mathcal{L}$ lets a feasible $Q$ move a vanishing amount of mass arbitrarily far out to raise its mean at almost no $\kl$ cost, whereas under $\logel$ this move is forbidden, so $\logel>0$ for every $m\neq m(\hat P_n)$. When $P_i \in \mathcal{P}_{[a,b]},$ we have $\logel\ge\kl_{\inf}^{\mathcal{P}_{[a,b]}}$ (for $m>m(\hat P_n)$, since the feasible set is smaller), so the $\logel$ index is never more optimistic than the $\kl$-UCB index. 

\paragraph{Slowly varying functions.} Finally, we now recall the notion of slowly varying functions, which we use to formulate some of our asymptotic results on regret tails. A function $f:(0,\infty)\to(0,\infty)$ is called slowly varying if $f(cx)/f(x)\to1$ as $x\to\infty$ for every $c>0$. Examples of such functions include $\log x$, $\log^2x$, $\log x\log\log x$. See \citet{Embrechts1997} for details. 

\paragraph{Notation.} For non-negative functions $f, g$, we write $f(n) = \mathcal{O}(g(n))$ if $\limsup_{n\to\infty} {f(n)}/{g(n)} <\infty$, $f(n) = \Omega(g(n))$ if $\liminf_{n\to\infty} {f(n)}/{g(n)} > 0$, $f(n) = o(g(n))$ if $\lim_{n\to\infty} {f(n)}/{g(n)} = 0$, and $f(n) = \omega(g(n))$ if $\lim_{n\to\infty} {f(n)}/{g(n)} = \infty$. 

\section{A false start: naive empirical-likelihood UCB}\label{sec:false_start}
Increasing the exploration rate of UCB-type algorithms can improve robustness and regret-tail behavior, at the cost of larger expected regret \citep{audibert2009exploration,simchi2023regret,fan2025fragility}. For scale-free non-parametric bandits, \citet{ashutosh2021bandit} study super-logarithmic exploration. However, applying this idea directly to $\kl$-UCB requires an informative choice of ambient class $\mathcal{L}$. As discussed in \Cref{sec:preliminaries}, taking $\mathcal{L}=\mathcal{P}_{\exp}$ gives $\kl_{\inf}^{\mathcal{P}_{\exp}}(\hat{P}_n,m)=0$ for every finite $m$, making the resulting index uninformative.

EL offers a data-adaptive alternative. Its candidate distributions are supported on the observed rewards, so computing the index requires no distributional knowledge. Motivated by the structural similarity between $\logel$ and $\kl_{\inf}$ discussed in \Cref{sec:preliminaries}, a natural approach is to replace $\kl_{\inf}$ with $\logel$ while retaining an inflated exploration threshold. This leads to the following naive EL-UCB policy.

\paragraph{Naive EL-UCB:} Let $g:\naturals\to(0,\infty)$ satisfy $g(t)=\omega(\log t)$. Pull each arm once. For $t\ge K+1$, let $A_t\in\arg\max_{i\in[K]}~U_i(t)$, where $U_i(t):=\sup\{m\in\reals: N_i(t)\logel(\hat{P}_i(t),m)\le g(t)\}.$
 \begin{remark}\label{rem:ELFailure}
     This policy incurs linear expected regret on a bounded two-armed instance. To see this, fix $\epsilon \in (0,1)$, and consider $P_1 = (1-\epsilon)\,\delta_0 + \epsilon\,\delta_{2/\epsilon}$, and $P_2 = \delta_1$. The arm means are $m(P_1) =  2$ and $m(P_2) = 1$, so arm $1$ is optimal and $\Delta_2 = 1$. Since arm $2$ always returns $1$, we have $\hat{P}_2(t) = \delta_1$ and $U_2(t) = 1$ for all $t \ge 3$ (Definition~\ref{def:logel}). However, the first pull of arm $1$ returns $0$ with probability $1-\epsilon$. On this event, after both arms have been pulled once, we have $\hat{P}_1(3) = \delta_0$, and hence $U_1(3) = 0 < 1 = U_2(3)$. The policy therefore plays arm $2$ at $t=3$, leaving the index of arm $1$ unchanged. This repeats at every subsequent round, resulting in $N_2(T+1) = T-1$ on this event, and consequently,$\expec{R(T)} \ge (1-\epsilon)(T-1)$ for all $T \ge 2$.
     
 \end{remark}
The example in Remark~\ref{rem:ELFailure} above, in fact, shows that no choice of exploration function $g$ can repair the failure. It is a symptom of the \emph{convex hull condition} that naive EL faces (cf. \citet[Section 10.4]{Owen2001}), which occurs because $\logel$ places an infinite penalty on means outside the empirical convex hull. To build a sound scale-free algorithm, we therefore incorporate forced exploration, which ensures that sampling continues across all arms and prevents the index from collapsing.

\section{Algorithm}\label{sec:algorithm}
We present \texttt{EL-F} in Algorithm~\ref{alg:el_ucb}. The algorithm interleaves forced exploration with selections based on an EL index: it samples an arm whose pull count is below $g(t)$ whenever such an arm exists (forced exploration phase); otherwise, it selects the arm with the largest index (UCB phase). The parameter $\crad \ge 0$ controls the optimism of the index. 
\begin{algorithm}[h]
\caption{\texttt{EL-F:} Empirical Likelihood UCB with Forced Exploration}
\label{alg:el_ucb}
\KwIn{$K$, non-decreasing exploration function $g:[1,\infty) \to (0,\infty)$, decision radius $\crad \ge 0$}
\KwInit{$N_i(1) = 0$, $\hat{P}_i(1) = 0$, and $\indexel_i(1) = +\infty \ \forall \ i\in [K].$}
\For{$t = 1, 2, \dots$}{
    \eIf{$\exists \, i \in [K]$ such that $N_i(t) < g(t)$}{
        Select $A_t\in\{i\in[K]:N_i(t)<g(t)\}$ \Comment{Choose least-sampled deficient arm.}
    }{
        Select $A_t \in \operatorname*{argmax}_{i \in [K]} \indexel_i(t)$ \Comment{Break ties arbitrarily.}
    }
    
    Observe reward $X_t$\;
    
    Update statistics of arm $A_t$:\;
    \begin{itemize}
        \item[] $\hat{P}_{A_t}(t+1) \gets \dfrac{N_{A_t}(t)\,\hat{P}_{A_t}(t) + \delta_{X_t}}{N_{A_t}(t)+1};$\; $N_{A_t}(t+1) \gets N_{A_t}(t) + 1$
        \item[] $\indexel_{A_t}(t+1) \gets \sup \{ m \in \mathbb{R} : N_{A_t}(t+1) \cdot \logel(\hat{P}_{A_t}(t+1),m) \le \crad \}$
    \end{itemize}
    
    \For{$j \neq A_t$}{
        $N_j(t+1) \gets N_j(t);$\;
        $\hat{P}_j(t+1) \gets \hat{P}_j(t);$\;
        $\indexel_j(t+1) \gets \indexel_j(t)$\;
    }
}
\end{algorithm}

A few remarks are in order. First, note that since $g(t)>0$, the UCB phase is entered only when $N_i(t)\ge g(t)>0$ for every arm $i$, every index used in an index comparison is computed from at least one sample. Since with one sample, or several identical samples, the EL index equals the common observed value (Definition~\ref{def:logel}), the initial values $+\infty$ of the EL index never influence arm selection. Next, since $\crad$ is fixed, only the index of the sampled arm needs updating at any time. Finally, note that we retain the term {UCB phase} by analogy with $\kl$-UCB, but the index only supplies a local optimism and is not calibrated to provide a confidence guarantee.

\begin{remark}\label[remark]{rem:imp_control_termB}
For every arm $i\in[K]$ and round $t\ge K+1$ with $N_i(t)\ge1$, $\logel(\hat{P}_i(t),\hat{m}_i(t))=0.$
Thus, the empirical mean is feasible in the definition of the index, giving $\indexel_i(t)\ge\hat{m}_i(t)$. Consequently, a lower deviation of the index below any threshold requires the empirical mean to fall below that threshold. We use this to control lower deviations of the index in our analysis.
\end{remark}

We note that forced exploration alone is sufficient for our guarantees in \Cref{sec:analysis,sec:tails}. Setting $\crad=0$ reduces the index for each arm to the empirical mean of that arm, and the $\mathcal{O}(g(T))$ expected-regret guarantee continues to hold under the assumptions of \Cref{thm:expected_regret}. This greedy variant is exactly the $g$-Forcing policy of \citet{Cowan_Katehakis_2020}, who establish almost-sure regret asymptotics for this policy, which our expected-regret and tail bounds complement. It is also related to \texttt{FE} \citep{han2024forced}, although the algorithms differ in how they trigger exploration. However, a positive radius introduces local optimism between forced pulls and can encourage an arm to be revisited before its next forced-exploration pull. We demonstrate below an instance wherein this local optimism provably improves performance. 
\begin{lemma}
    \label{lem:Crad-better}
    Consider a two-armed bandit instance such that $P_1$ is stochastic, $P_2=\delta_{\mu_2}$
    is deterministic, and $m(P_1)>\mu_2$. Keeping all else fixed, let $N^C_1(T),N^0_1(T)$ denote the number of pulls of arm $1$ in \Cref{alg:el_ucb} with $\crad = C>0$ and $\crad = 0$, respectively, with corresponding pseudo-regrets $R^C(T), R^0(T)$. 
    Then, for every $T\geq 1$,
    $N_1^C(T+1)\geq N_1^0(T+1)$, and
    $R^C(T)\leq R^0(T)$.
\end{lemma}

\begin{proofsketch}
    Let $\indexel_{i,\crad}(t)$ denote index of arm $i$ at time $t$ in \Cref{alg:el_ucb} with the appropriate choice of $\crad$. Then $\indexel_{2,C}(t)=\indexel_{2,0}(t) = \mu_2$ and $\indexel_{1,C}(t)\geq \indexel_{1,0}(t) = m(\hat{P}_1(t))$, whenever arm $1$ is chosen under $\crad = 0$ it is also pulled for $\crad = C$. See \Cref{app:crad_better} for a proof.
\end{proofsketch}

Our ablation study in \Cref{app:ablation} shows that this optimism can reduce the frequency of prolonged suboptimal play in some instances. In fact, we demonstrate that this optimism enables the index to adapt to the empirical variance (\Cref{lem:index_mean_difference_bound}), yielding a more adaptive index than the empirical mean alone. As noted in \Cref{rem:imp_control_termB}, the index always upper bounds the empirical mean. 

\section{Almost Sure and Expected Regret Analysis}\label{sec:analysis}
In this section, we analyze the \emph{almost sure and expected pseudo-regret} of \Cref{alg:el_ucb}. We state the key result of this section, followed by the key technical highlight of the proof.

\begin{theorem}\label{thm:expected_regret}
    Let $g:[1,\infty)\to(0,\infty)$ be non-decreasing with $g(t)\to\infty$ as $t\to \infty$, and $\Delta := \sum\nolimits_{i:\Delta_i>0}\Delta_i$.
\begin{enumerate}
    \item If additionally, $g(t+1)-g(t)\to0$ and $\expecover{P_i}{|X|} < \infty \ \forall \ i\in [K]$, then for $\delta\in(0,1)$, for all sufficiently large $T$, \Cref{alg:el_ucb} satisfies
    \begin{equation}
    \bigl(g(T)-\delta\bigr)\Delta \ \le\ R(T)\ \le\ \lceil g(T)\rceil\Delta \quad \text{ a.s.},
    \end{equation}
    which implies $\lim_{T\to\infty}{R(T)}/{g(T)}=\Delta$, almost surely.
    \item If $g(t)=\omega(\log t)$ and $g(t)\le t-1$ for $t\ge2$, then for every $\underline{P}\in\mathcal{P}_{\exp}^K$, \Cref{alg:el_ucb} satisfies, for every $T\in\naturals$,
    $$
    \expec{R(T)} \leq g(T) \Delta + C_{\underline{P}},
    $$
    where $C_{\underline{P}}<\infty$ is an instance-dependent constant, independent of $T$, explicitly presented in the proof (\Cref{eq:c_p_exact}).
\end{enumerate}
\end{theorem}

\begin{proofsketch}
    We show that the EL index is a strongly consistent estimator of the true mean (\Cref{lem:index_consistency}), following which Part (1) essentially follows the proof of \citet[Proposition 1]{Cowan_Katehakis_2020}. For Part (2), our proof relies on decomposing the number of pulls of suboptimal arms in the forced exploration and UCB phases. In particular, for a sub-optimal arm $i$, we bound this as 
    $\expec{N_i(T+1)} \leq \lceil g(T)\rceil +\textnormal{Term A} + \textnormal{Term B},$
    where Term A arises from the index of arm $i$ overshooting its true mean, and Term B arises from the empirical mean of arm $1$ underestimating its true mean. The latter term is easily shown to be $\mathcal{O}(1)$ using a union bound and standard sub-exponential concentration inequalities. 
    
    For Term A, we use an upper bound on the index of arm $i$ (\Cref{lem:index_mean_difference_bound}), which allows us to control the index deviation. Following this, we control the sample variance in terms of the sample mean and the max deviation of a certain sub-exponential random variable (\Cref{lem:sample_variance_bound}). These different terms are then bounded using standard tools to complete the proof in \Cref{app:expected_regret_proofs}.
\end{proofsketch}


A key intermediate technical step used in both parts of the proof above is to control for deviations of the index from the sample mean. To this end, we present the following lemma, which gives an upper bound on the index. See \Cref{app:main_text_lemma_proofs} for a proof of \Cref{lem:index_mean_difference_bound}.

\begin{lemma}\label[lemma]{lem:index_mean_difference_bound}
Fix an arm $i$ and a round $t$ with $N_i(t)=n\ge1$. Let $\hat m_i(t)$, $\hat\sigma^2_i(t)$ and
$X^{\max}_i(t)$ be the empirical mean, empirical variance and maximum of the $n$ rewards observed from arm $i$
before round $t$. Then
\[
0\;\le\;\indexel_i(t)-\hat m_i(t)\;\le\;\frac{2\crad}{n}\bigl(X^{\max}_i(t)-\hat m_i(t)\bigr)+\sqrt{\frac{2\crad}{n}\,\hat\sigma^2_i(t)} .
\]
\end{lemma}

\section{Pseudo-regret Tail Characterization}\label{sec:tails}
We now study the tails of the pseudo-regret $R(T)$, which describe how likely rare, catastrophic runs are; see \citet{audibert2009exploration, Salomon2011, fan2025fragility,simchi2023regret, simchi2025simple}. First, we demonstrate an upper bound on the tails of pseudo-regret for \texttt{EL-F}. Next, we show that these bounds are essentially tight in certain regimes, a notion we will make precise below. 

Note that forced exploration contributes at most $\Delta\lceil g(T)\rceil$ to $R(T)$, where $\Delta=\sum_{i:\Delta_i>0}\Delta_i$. The next result says that regret beyond this budget has a tail controlled by $g$ itself. 
 
\begin{theorem}[Regret tail bound]\label{thm:tail}
Let $g:[1,\infty)\to(0,\infty)$ be non-decreasing with $g(t)\le t-1$ for $t\ge2$ and $g(t)=\omega(\log t)$. Let $\underline{P}\in\mathcal{P}_{\exp}^K$. Then, there are instance-dependent constants $C_0<\infty$ and $C_{\min}>0$ such that the following hold.
\begin{enumerate}
    \item[(i)] For every $T\in\mathbb{N}$ and $x>\Delta\lceil g(T)\rceil$, writing $k_x=\big\lceil\frac{x-\Delta\lceil g(T)\rceil}{\Delta}\big\rceil$,
    $$
    \log \mathbb{P}\big(R(T)\ge x\big)\ \le\ \log((K-1)C_0k_x) -C_{\min}g(k_x);
    $$
    \item[(ii)] Moreover, if $g$ is slowly varying and $x_T=\omega(g(T))$, then
    $$
    \limsup_{T\to\infty}\frac{\log\mathbb{P}\big(R(T)\ge x_T\big)}{g(x_T)}\ \le\ -C_{\min}.
    $$
\end{enumerate}
\end{theorem}
 
\begin{proofsketch}
We first split $R(T)$ into a contribution from the forced exploration phase, which is at most $\Delta\lceil g(T)\rceil$, and that from the UCB-phase. For the latter, we take a union bound over suboptimal arms. Thus, it suffices to bound $\prob{N_i^U(T)\ge k}$, where $N_i^U(T)$ counts UCB-phase pulls for the sub-optimal arm $i$ till time $T$. At the $k$-th such pull of arm $i$, it has at least $k-1$ samples. Moreover, at that time, arm $1$ has at least $g(k)$ samples (by the forced-exploration), and the index of arm $i$ is at least that of arm $1$. As in \Cref{lem:terma_subexponential}, this forces one of four concentration failures: one for arm $1$, which is a geometric series starting at $g(k)$, and three for arm $i$, which decay exponentially in $k$. 
Each of these terms is bounded by $k\,e^{-C_{\min}g(k)}$.
Part (ii) follows by taking logarithms and using $g=\omega(\log t)$, monotonicity of $g$ and slow variation. The full proof is in \Cref{app:regret_tail_upper_bound}.
\end{proofsketch}

Thus, beyond the forced-exploration budget, the regret tails of \texttt{EL-F} are at most $\exp(-\Omega(g(x)))$. Since $g(t)=\omega(\log t)$, this bound decays faster than $x^{-p}$ for every $p>0$, while for sublinear $g$ it decays slower than exponentially. To place this rate within the landscape of existing results, we recall that no consistent policy can have exponentially decaying regret tails uniformly over the ambient class \citep[Theorem 1]{simchi2025simple}. Moreover, it is helpful to contrast our rate with policies that attain logarithmic regret: \citet[Theorem 1]{fan2025fragility} show that policies which attain the Lai-Robbins lower bound have tails of order $1/x$, and that inflating the exploration of $\kl$-UCB to $(1+b)\log t$ only improves the rate to $x^{-(1+b)}$ at the price of a larger constant in front of $\log T$. 

Next, we show that the upper bound in \Cref{thm:tail} is essentially tight. To do so, we will first show a lower bound on the tails of any bandit algorithm acting on a well-constructed Gaussian instance. Following this, we will show that on this instance the tails of \texttt{EL-F} exactly match the lower bound. The general lower bound result below provides a simplified proof of \citet[Theorem 1 (2)]{simchi2023regret} and improves the constants in their result.
 
\begin{proposition}[Lower bound]\label[proposition]{thm:simchi_lb}
Let $K=2$, $\underline{P}=(N(\mu_1,\sigma^2),N(\mu_2,\sigma^2))$ with $\Delta=\mu_1-\mu_2>0$, and $\underline{Q}=(N(\mu_{1,Q},\sigma^2),N(\mu_2,\sigma^2))$ with $\mu_{1,Q}<\mu_2$ and $\Delta_Q=\mu_2-\mu_{1,Q}$. Let $\pi$ be any policy with $\expecover{\underline{Q}}{R(T)}=o(T)$, and  $\limsup_Tx_T/T<\Delta$. Then
$$
\log\mathbb{P}_{\underline{P}}\big[R(T)>x_T\big]\ \ge\ -(1+o(1))\left[\frac{(\Delta+\Delta_Q)^2}{2\sigma^2}\,\expecover{\underline{Q}}{N_1(T+1)}+\log2\right].
$$
\end{proposition}

\begin{proofsketch}
     The two instances differ only in arm $1$, so using a chain rule for $\kl$ divergences, we obtain that the divergence between the two histories is $\expecover{\underline{Q}}{N_1(T+1)}(\Delta+\Delta_Q)^2/(2\sigma^2)$. Applying the data-processing inequality, followed by an algebraic inequality on the binary $\kl$ divergence completes the proof. The full proof is in \Cref{app:lower_bound_tails}. 
\end{proofsketch}

Using this, we characterize the exact exponent attained by our algorithm for linear deviations and demonstrate the tightness of our proposed algorithm. 
 
\begin{corollary}[Exact exponent and optimality]\label[corollary]{cor:gaussian}
Let $g$ be as in \Cref{thm:tail} and slowly varying. Let $\underline{P}$ and $\Delta$ be as in \Cref{thm:simchi_lb} and let $c\in (0,\Delta)$. Then, for \texttt{EL-F}
$$
\lim_{T\to\infty}\frac{\log\mathbb{P}_{\underline{P}}\big[R(T)>cT\big]}{g(T)}\ =\ -\frac{\Delta^2}{2\sigma^2}.
$$
\end{corollary}
 
\begin{proofsketch}
For the upper bound, we follow the proof of \Cref{thm:tail} with $x=cT$, now using Gaussian concentration for the optimal arm. For the lower bound, we essentially take $\Delta_Q \to 0$ in \Cref{thm:simchi_lb}. We note that the lower bound holds for all policies such that the expected number of pulls of the sub-optimal arm is at most $(1+o(1))g(T)$. The full proof is in \Cref{app:optimality_tails}.
\end{proofsketch}
 
Thus, the constant in the exponent is exact: a catastrophic run of \texttt{EL-F} is one in which the optimal arm looks $\Delta$ worse than it is, for about $g(T)$ samples. The probability of this happening is at most $\exp(-(1+o(1))g(T)\Delta^2/(2\sigma^2))$. In fact, no policy with the same expected number of pulls of the suboptimal arm can make this rarer.

\section{Experiments}\label{sec:experiments}

We empirically compare \Cref{alg:el_ucb} with existing scale-free bandit algorithms in terms of mean and median of pseudo-regret across independent runs. We consider two super-logarithmic exploration schedules,
$g_1(t)=\frac{1}{K}\log t\log\log t$ and
$g_2(t)=\frac{1}{K}(\log t)^2$,
with the initial-round conventions specified in Appendix~\ref{app:experiments}. For \texttt{EL-F}, we consider $\crad\in\{0,3\}$ to examine the effect of local optimism in addition to forced exploration.



\begin{itemize}[nosep]
    \item \texttt{R-UCB} \citep{ashutosh2021bandit} plays $\operatorname{argmax}_i \, \hat{m}_i(t) + \sqrt{\varphi(t)/N_i(t)}$ for an exploration function $\varphi(t)=\omega(\log t)$. We use $\varphi\in\{g_1,g_2\}$.
    
    \item \texttt{FE} \citep{han2024forced} alternates forced exploration with greedy pulls of the empirically best arm. Its exploration schedule is set by a sequence $f(\cdot)$; we use the three sequences suggested by the authors, $f(r)=\sqrt{T}$, $f(r)=r$ and $f(r)=e^{r/\log T}$. The first and third require the horizon $T$ in advance, whereas $f(r)=r$ can be run anytime.

    \item \texttt{RDS} \citep{baudry2021optimality} is a Dirichlet re-sampling algorithm whose exploration is governed by a sequence $\rho_n$. We use $\rho_n = \sqrt{\log(1+n)}$, the choice in their experiments, and the more exploratory $\rho_n = \log(1+n)$.
\end{itemize}

\paragraph{Instances.} We evaluate all algorithms on six bandit instances: the four instances described in \Cref{tab:instances} and unscaled versions of the two Beta instances, with rewards in $[0,1]$. We ran each (algorithm, instance) pair for $T = 10{,}000$ rounds and $1{,}024$ independent runs, and report the mean and median of the pseudo-regret $R(T) = \sum_{a}\Delta_a N_a(T+1)$ in \Cref{tab:regret-summary}. We omit \texttt{R-UCB} and \texttt{FE} in \Cref{fig:regret_combined} as their regret scales are significantly larger than \texttt{EL-F} and \texttt{RDS} (see Figure~\ref{fig:regret_curves_all} in Appendix~\ref{app:all_algo_experiment}). 



\paragraph{Observations.} 


With $\crad > 0$, \texttt{EL-F} performs competitively with \texttt{RDS}, generally achieving smaller median regret and a comparable mean regret (illustrated in \Cref{tab:regret-summary} and \Cref{fig:regret_combined}).
The distinction between mean and median regret
is pronounced on more challenging instances like Beta with small gaps. This gap between its median and mean stems from heavier regret tails, which are observed to decrease with increased exploration and a larger optimism radius. 
\texttt{R-UCB} takes this to an extreme, achieving the lowest median regret on 4 out of 6 settings but incurring much larger mean regret (\Cref{tab:regret-summary} and \Cref{fig:regret_curves_all}). Moreover, we observe that \texttt{R-UCB} is not invariant under reward scaling and behaves differently on the $[0,1]$ and $[0,50]$ versions of the same instance (see Table~\ref{tab:regret-summary}, and \Cref{para:effect-of-reward-scaling} for a detailed discussion). \texttt{FE} exhibits consistent regret across runs, but often incurs higher exploration costs than the best configurations
of \texttt{EL-F} and \texttt{RDS}.



The effects of forced exploration and local optimism are
particularly evident in instances with frequent high-regret
runs (\Cref{tab:regret-summary,tab:crad-ablation,tab:trapped-all}).
Increasing exploration from $g_1$ to $g_2$ reduces the
frequency of such runs and improves mean regret on the
small-gap Beta and exponential instances, but incurs
higher exploration costs on easier instances.
Similarly, increasing $\crad$ reduces mean regret on
the more challenging instances, while typically increasing
median regret. 
Based on these observations, we recommend choosing $\crad=3$, which strikes a good balance between mean and median regret across the instances considered. As shown in Appendix~\ref{app:ablation}, a slowly increasing radius $\crad(t)=\log\log t + 3$ can further improve mean regret on some instances, although the gains are not uniform.

\noindent

\section{Conclusion \& Discussion}
We studied scale-free regret minimization on rich ambient classes, where $\kl_{\inf}^{\mathcal{L}}=0$ rules out logarithmic expected regret. We proposed \texttt{EL-F}, an anytime EL extension of the $g$-Forcing policy that requires no knowledge of reward scales or tail parameters. We characterized its pseudo-regret at three levels. For integrable rewards, \texttt{EL-F} preserves the almost-sure asymptotics of $g$-Forcing. On every sub-exponential instance and for any $g=\omega(\log t)$, its expected regret is at most $g(T)\sum_i\Delta_i$ plus an instance-dependent constant, and its regret beyond the forced-exploration budget has tails $\exp(-\Omega(g(x)))$. For slowly varying $g$, the exponent of linear regret deviations is exact on two-armed Gaussian instances with common variance. These guarantees hold for every fixed radius $\crad\ge0$, and thus also cover the greedy variant.

An important open direction is to better understand the role of local optimism. Our analysis shows that a positive radius preserves the leading regret, and our coupling argument shows that it never hurts in two-armed bandits with a deterministic suboptimal arm. Experimentally, it reduces prolonged suboptimal allocation on challenging instances. However, a general theory of when and by how much $\crad>0$ helps remains open. Finally, the empirical-likelihood statistic requires no moment conditions, and sub-exponentiality enters our analysis only through concentration. This makes \texttt{EL-F} a natural candidate for scale-free learning with heavy-tailed rewards \citep{genalti2025open, genalti2026parameter}, which we leave for future work.

\begin{table}[t]
\centering
\small
\resizebox{\ifdim\width>\textwidth\textwidth\else\width\fi}{!}{%
\begin{tabular}{lcll}
\toprule
\textbf{Instance} & $K$ & \textbf{Parameters} & \textbf{Means} $(m_a)_a$ \\
\midrule
Beta (wide gaps)  & 4 &
$(\alpha_a,\beta_a)\in\{(1,9),(3,7),(5,5),(8,2)\}$ &
$(5,\,15,\,25,\,40)$ \\[2pt]
Beta (small gaps) & 5 &
$(\alpha_a,\beta_a)\in\{(5,10),(6,9),(7,8),(8,7),(9,6)\}$ &
$(16.67,\,20,\,23.33,\,26.67,\,30)$ \\[2pt]
Gaussian          & 4 &
$\sigma = 10$ &
$(5,\,15,\,25,\,40)$ \\
Exponential       & 4 &
rate $1/m_a$ &
$(5,\,15,\,25,\,40)$ \\
\bottomrule
\end{tabular}%
}
\caption{Instances used in experiments. Beta instances are scaled by 50 (arm $a$ has law $50\cdot\mathrm{Beta}(\alpha_a,\beta_a)$); Gaussian arms are $N(m_a,\sigma^2)$ and exponential arms are $\mathrm{Exp}(1/m_a)$.}
\label{tab:instances}
\end{table}

\begin{figure}[h]
    \centering

    \includegraphics[width=0.9\textwidth]{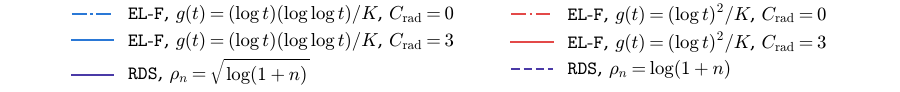}\\[2pt]

    \includegraphics[width=\linewidth]{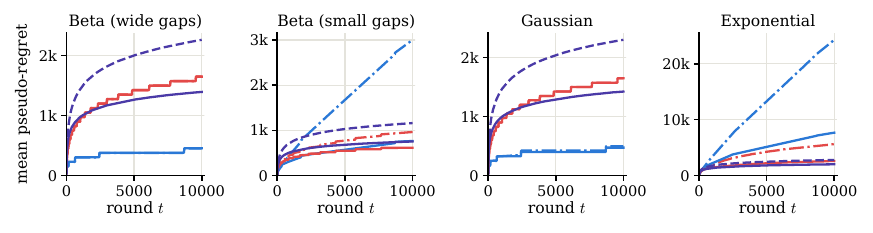}\\
    {\small (a) Mean pseudo-regret over $1{,}024$ runs.}
    \label{fig:regret_combined_preg}\\[3pt]

    \includegraphics[width=\linewidth]{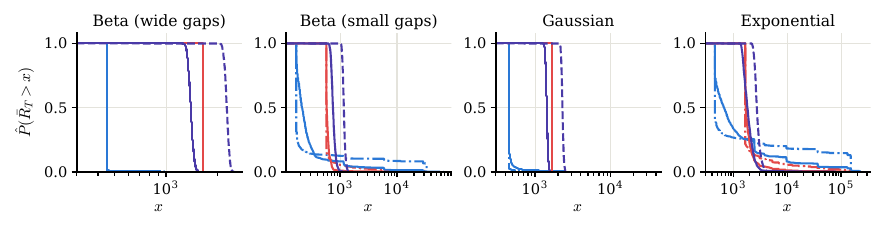}\\
    {\small (b) Empirical tail $\hat P(R_T > x)$ of the pseudo-regret ($x$ on a logarithmic scale).}
    \label{fig:regret_combined_tails}

    \caption{Pseudo-regret performance over $1{,}024$ runs.}
    \label{fig:regret_combined}
\end{figure}

\begin{table}[h]
\centering
\caption{Pseudo-regret at $T = 10{,}000$ over 1{,}024 independent runs, where $g_1(t) = \frac{1}{K}(\log t )\,(\log\log t)$ and $g_2(t) = \frac{1}{K}(\log t)^2$, with $K$ the number of arms.}
\label{tab:regret-summary}
\resizebox{\ifdim\width>\textwidth\textwidth\else\width\fi}{!}{%
\begin{tabular}{@{}l*{12}{r}@{}}
\toprule
Policy & \multicolumn{4}{c}{Beta (wide gaps)} & \multicolumn{4}{c}{Beta (small gaps)} & \multicolumn{2}{c}{Gaussian} & \multicolumn{2}{c}{Exponential} \\
\cmidrule(lr){2-5} \cmidrule(lr){6-9} \cmidrule(lr){10-11} \cmidrule(lr){12-13}
 & \multicolumn{2}{c}{$[0,1]$} & \multicolumn{2}{c}{$[0,50]$} & \multicolumn{2}{c}{$[0,1]$} & \multicolumn{2}{c}{$[0,50]$} & & & & \\
\cmidrule(lr){2-3} \cmidrule(lr){4-5} \cmidrule(lr){6-7} \cmidrule(lr){8-9}
 & mean & med. & mean & med. & mean & med. & mean & med. & mean & med. & mean & med. \\
\midrule
\texttt{EL-F}, $g_1$, $C_{\mathrm{rad}} = 0$ & \textbf{9.06} & \textbf{9.00} & \textbf{453} & 450 & 60.0 & \textbf{3.33} & 3{,}001 & 167 & 498 & 450 & 24{,}255 & 450 \\
\texttt{EL-F}, $g_1$, $C_{\mathrm{rad}} = 3$ & 9.10 & \textbf{9.00} & 455 & 450 & 15.3 & 4.20 & 765 & 210 & \textbf{474} & 450 & 7{,}672 & 660 \\
\texttt{EL-F}, $g_2$, $C_{\mathrm{rad}} = 0$ & 33.0 & 33.0 & 1{,}650 & 1{,}650 & 19.3 & 11.3 & 964 & 567 & 1{,}650 & 1{,}650 & 5{,}611 & 1{,}650 \\
\texttt{EL-F}, $g_2$, $C_{\mathrm{rad}} = 3$ & 33.0 & 33.0 & 1{,}650 & 1{,}650 & \textbf{12.3} & 11.5 & \textbf{614} & 573 & 1{,}650 & 1{,}650 & 2{,}607 & 1{,}695 \\
\midrule
\texttt{RDS}, $\rho_n = \sqrt{\log(1+n)}$ & 27.9 & 27.8 & 1{,}395 & 1{,}390 & 15.1 & 14.9 & 754 & 743 & 1{,}425 & 1{,}425 & \textbf{2{,}018} & 1{,}810 \\
\texttt{RDS}, $\rho_n = \log(1+n)$ & 45.3 & 45.3 & 2{,}264 & 2{,}265 & 23.2 & 23.1 & 1{,}161 & 1{,}153 & 2{,}302 & 2{,}295 & 2{,}804 & 2{,}710 \\
\midrule
\texttt{R-UCB}, $\varphi = g_1$ & 31.8 & 31.8 & 833 & \textbf{75.0} & 88.7 & 88.6 & 6{,}554 & \textbf{46.7} & 3{,}313 & \textbf{75.0} & 56{,}385 & \textbf{120} \\
\texttt{R-UCB}, $\varphi = g_2$ & 116 & 116 & 680 & \textbf{75.0} & 260 & 259 & 4{,}952 & 50.0 & 2{,}510 & \textbf{75.0} & 55{,}007 & \textbf{120} \\
\midrule
\texttt{FE}, $f(r) = e^{r/\log T}$ & 96.0 & 96.0 & 4{,}800 & 4{,}800 & 42.8 & 42.7 & 2{,}140 & 2{,}133 & 4{,}800 & 4{,}800 & 4{,}855 & 4{,}800 \\
\texttt{FE}, $f(r) = \sqrt{T}$ & 149 & 149 & 7{,}473 & 7{,}460 & 71.5 & 66.5 & 3{,}573 & 3{,}323 & 7{,}514 & 7{,}460 & 10{,}049 & 7{,}520 \\
\texttt{FE}, $f(r) = r$ & 210 & 210 & 10{,}500 & 10{,}500 & 93.6 & 93.3 & 4{,}681 & 4{,}667 & 10{,}500 & 10{,}500 & 10{,}659 & 10{,}500 \\
\bottomrule
\end{tabular}%
}
\end{table}

\section{AI Disclosure}
The core results in this work were obtained without using AI tools. The authors acknowledge the use of GPT-6 Astra/Sol for verifying the proofs. It also suggested Lemma~\ref{lem:sample_variance_bound}. Opus-5.5 was also used to help improve the presentation. The authors acknowledge Claude Code's assistance in implementing the experiments. The authors have verified all code, statements, and proofs, and take full responsibility regarding the correctness and clarity of the work.

\section*{Acknowledegments} SA acknowledges the support from ANRF’s grant ANRF/ECRG/2025/000560/ENS and the generous support from the Pratiksha Trust, Bangalore, through the Young Investigator Award. This work was also supported by the Centre for Networked Intelligence (a Cisco CSR initiative) at the Indian Institute of Science, Bengaluru.

\bibliography{refs}

\newpage

\appendix
\crefalias{section}{appendix}
\crefalias{subsection}{subappendix}
\section{Related Work.}\label{sec:related}
In this appendix, we discuss the related works in detail.


\paragraph{Impossibility and scale-free algorithms.}
Within parametric families, scale-free logarithmic regret is attainable: \citet{cowan2018normal} give an inflated-sample-mean index policy that is asymptotically optimal for Gaussian rewards with unknown means and variances, and \citet{cowan2015asymptotically} do the same for uniform rewards with unknown supports. Both their guarantees rely on the parametric form and do not extend to the non-parametric classes considered here. 

On sufficiently rich classes $\kl_{\inf}^{\mathcal{L}}=0$, so logarithmic regret is impossible \citep{lattimore2017scale, pmlr-v117-agrawal20a, hadiji2023adaptation, ashutosh2021bandit}. 
\citet{lattimore2017scale} gives a scale-free algorithm under a known bound on the kurtosis and proposes inflating the exploration rate to trade a slightly super-logarithmic regret for adaptivity. \citet{ashutosh2021bandit} realize this with \texttt{R-UCB}, which inflates the exploration rate of UCB to $g(t)=\omega(\log t)$ and obtains $\mathcal{O}(g(T))$ regret on light-tailed instances, together with a truncation-based variant for heavy tails. \citet{baudry2021optimality} obtain an analogous trade-off with re-sampling (Dirichlet sampling, \texttt{RDS}). By \citet[Theorem 1]{hadiji2023adaptation}, a scale-free policy whose distribution-dependent regret is $\mathcal{O}(g(T))$ on bounded instances cannot have distribution-free regret $o(T/g(T))$; like \texttt{R-UCB} and \texttt{RDS}, \texttt{EL-F} gives up worst-case regret in exchange for near-logarithmic instance-dependent regret.

\paragraph{Forced exploration.}
The $g$-Forcing policy of \citet{Cowan_Katehakis_2020} samples an arm whose pull count falls below $g$ and otherwise plays greedily, and for slowly growing $g$ its pseudo-regret is of order $g(n)$ almost surely, with almost-sure bounds on the remainder term. With $\crad=0$, \texttt{EL-F} essentially reduces to this policy. Relative to their work, we provide finite-time expected-regret bounds with explicit constants, quantitative regret-tail bounds, an exact tail exponent on Gaussian instances, and the empirical-likelihood index for $\crad>0$. \citet{han2024forced} interleave forced exploration with greedy choices and obtain near-logarithmic expected regret for sub-Gaussian rewards with unknown variance proxy. Their exploration is triggered by the time since an arm's last pull, and their near-logarithmic schedules take the horizon as input. Unlike either of these works, we analyze the regret \emph{tails}.

 
\paragraph{Heavy-tailed bandits.}
For rewards with only bounded moments, robust estimators give optimal regret when the moment bound is known \citep{bubeck2013bandits, agrawal2021regret}. Recent work studies parameter-free variants by imposing relaxing assumptions that allow them to get around the impossibility result. These include examples such as truncated non-positivity \citep{genalti2024varepsilon} and symmetry \citep{tamas2024data}. \cite{genalti2025open} stated several open problems on this front, including assumption-free adaptation, as well as understanding how different relaxing assumptions relate to one another. Recent work by \cite{genalti2026parameter} makes progress on the former of these problems. While our current setting is light-tailed, we believe it would be an interesting avenue of future work to extend our methods to the heavy-tailed setting.
 
\paragraph{Empirical likelihood.}
We note that the empirical likelihood method has been applied successfully in the contextual bandit literature \citep{karampatziakis2020empirical, ouyang2026bayesian}. Empirical-likelihood confidence bounds have also been used in distributionally robust optimization \citep{duchi2021statistics}.
 
\paragraph{Regret tails.}
The distribution of the pseudo-regret was studied by \citet{audibert2009exploration, Salomon2011}, who show that UCB-type algorithms can have polynomially decaying regret tails. \citet{fan2025fragility} show that this heavy-tail behaviour is intrinsic to algorithms that match the Lai-Robbins bound. \citet{simchi2023regret} characterize the trade-off between expected regret and tail risk: their lower bound ties the tail of a policy to its expected regret on an alternative instance, and they give an anytime UCB policy that attains the matching orders, also for sub-exponential noise, under the assumption that the means lie in a known range and with a hyper-parameter in the sub-exponential case. \citet{simchi2025simple, Panda2026} study tail-safe policies and tail characterizations for general reward models. In \Cref{sec:tails} we prove a lower bound by a change-of-measure argument of the same type, with an explicit constant, and use it to show that the exponent of our tail bound is exact on Gaussian instances (\Cref{cor:gaussian}). Our tail bound $\exp(-\Omega(g(x)))$ holds for every sub-exponential instance without a known range for the means and without distribution-specific parameters.

\section{Deferred proofs}\label{app:proofs}
This appendix contains the full proofs of the results appearing in the main text.

\subsection{Preliminary lemmas}\label{app:crad_better}

\begin{proofof}{Lemma~\ref{lem:Crad-better}. }
We couple the two runs of the algorithm (namely, with $\crad = 0$ and $\crad=C>0$) by using the same reward stream for each arm (recall the reward-stack model specified in \Cref{sec:preliminaries}). Write $N_i^C(t)$ and $N_i^0(t)$ for the respective pull counts before round $t$. We prove by induction on $t$ that
\[
N_1^C(t)\geq N_1^0(t).
\]
The claim holds at $t=1$, since both counts are zero.

Suppose it holds at round $t$. The inequality can fail at round $t+1$ only if the runs pull the following arms at time $t$ and satisfy:
\[
N_1^C(t)=N_1^0(t),
\qquad
A_t^C=2,\quad A_t^0=1.
\]
We show that this is impossible.

Since the two algorithms have made the same number of pulls of arm $1$ by round $t$, and both have made exactly $t-1$ pulls in total, their pull counts agree for both arms:
\[
N_i^C(t)=N_i^0(t),\qquad i=1,2.
\]
By the coupling, their empirical distributions are consequently identical.

First, consider the forced-exploration phase. The set of deficient arms is identical in the two runs. Since the same least-sampled rule and tie-breaking convention are used, both
algorithms select the same arm. Thus the claimed divergence cannot occur in this phase.

Next, consider the UCB phase. Since arm $2$ is deterministic, its empirical distribution
is always $\delta_{\mu_2}$. By the definition of the empirical-likelihood statistic,
\[
\indexel_{2,C}(t)
=\indexel_{2,0}(t)=\mu_2.
\]
From \Cref{rem:imp_control_termB}, we have that the empirical mean lower bounds the index, i.e
\[
\indexel_{1,C}(t)
\geq \hat m_1(t)
=\indexel_{1,0}(t).
\]
If the zero-radius algorithm selects arm $1$, then, by its selection rule,
\[
\indexel_{1,0}(t)\geq
\indexel_{2,0}(t)=\mu_2.
\]
It follows that
\[
\indexel_{1,C}(t)
\geq \indexel_{1,0}(t)
\geq \mu_2
=\indexel_{2,C}(t).
\]
Therefore, the radius-$C$ algorithm must also select arm $1$, using the stipulated tie-breaking rule. This contradicts
$A_t^C=2$ and $A_t^0=1$.

Hence the induction is complete, and
\[
N_1^C(t)\geq N_1^0(t)
\qquad\text{for all }t\geq 1.
\]

Finally, since arm $1$ is optimal and
$\Delta=m(P_1)-m(P_2)>0$, the pseudo-regret of either algorithm is
\[
R^C(T)=\Delta N_2^C(T+1)
=\Delta\bigl(T-N_1^C(T+1)\bigr).
\]
Thus,
\[
N_1^C(T+1)\geq N_1^0(T+1)
\implies R^C(T)\leq R^0(T),
\]
as claimed.
\end{proofof}

\subsection{Almost Sure and Expected Regret Analysis}\label{app:expected_regret_proofs}
In this section, we prove \Cref{thm:expected_regret}, our main result concerning the pseudo-regret random variable induced by \Cref{alg:el_ucb}. Following the proof of \Cref{thm:expected_regret}, we prove the intermediate lemmas and results used, including \Cref{lem:logel_lower_bound,lem:index_mean_difference_bound}.

\begin{proofof}{\Cref{thm:expected_regret}. }
    \textit{Part (1).} The proof effectively follows the proof of Proposition 1 in \citet{Cowan_Katehakis_2020}. Our algorithm can be seen as the $g$-forcing policy that simply uses the EL index instead of the sample mean in the UCB phase (called ``play the winner" phase in their proof). The same lower bound on the number of sub-optimal pulls holds, since it only depends on $g$ and not the nature of the UCB phase. On the other hand, using strong consistency of the EL index (\Cref{lem:consistency_of_index}) instead of the strong law of large numbers in the paragraph preceding \citet[Equation 44]{Cowan_Katehakis_2020} and following the same argument provides the upper bound.
    
    \textit{Part (2).} Let $T$ be the time horizon and let $S_X \subset [T]$ be the time indices spent in the forced exploration phase, and $S_U = [T] \setminus S_X$ be the time indices spent in the UCB phase. Then, for any $i\in[K]$:
\begin{align*}
    N_i(T+1) &= \sum_{t = 1}^T \indicator{A_t = i} \\
    &= \sum_{t = 1}^T \indicator{A_t = i, t\in S_X} + \sum_{t = 1}^T \indicator{A_t = i, t\in S_U}  \\
    & \leq \lceil g(T)\rceil + \sum_{t = 1}^T \indicator{A_t = i, t\in S_U},
\end{align*}
where the inequality follows from the definition of the forced exploration phase: a forced pull of arm $i$ at round $t$ requires $N_i(t)<g(t)\le g(T)$, so there are at most $\lceil g(T)\rceil$ of them. No separate initialization cost arises, because the forced phase itself initializes every arm. Now, we focus on controlling the latter term in expectation for some suboptimal arm $i$, i.e. where $\Delta_i = m(P_1) - m(P_i)> 0$. 

Choose $\epsilon \in (0, \Delta_i)$:
    \begin{align}
        \expec{\sum_{t = 1}^T \indicator{A_t = i, t\in S_U}}  &= \sum_{t = 1}^T \prob{A_t = i, t\in S_U} \nonumber\\
        &= \sum_{t = 1}^T \prob{\indexel_i(t) \geq \indexel_1(t), A_t = i, t\in S_U} \nonumber\\
        &= \sum_{t = 1}^T \prob{\indexel_i(t) \geq \indexel_1(t), \indexel_1(t) \geq m(P_1) - \epsilon, A_t = i,  t\in S_U}\nonumber \\
        &\quad+\! \sum_{t = 1}^T \prob{\indexel_i(t) \geq \indexel_1(t), \indexel_1(t) < m(P_1) - \epsilon, A_t = i, t\in S_U}\nonumber \\
        &\leq \underbrace{\sum_{t = 1}^T \prob{\indexel_i(t) \geq m(P_i) +\Delta_i - \epsilon, A_t = i, t\in S_U}}_{\text{Term A}} \label{eq:termA}\\
        &\qquad + \underbrace{\sum_{t = 1}^T \prob{\hat{m}_1(t) < m(P_1) - \epsilon, t\in S_U}}_{\text{Term B}} \label{eq:termB},
    \end{align}
    where the last inequality follows from \Cref{rem:imp_control_termB} and $\hat{m}_1(t)$ is the empirical mean of arm $1$ from the $N_1(t)$ samples observed till time $t$. Next, we show that as $T$ grows larger, the dominant cost is solely of the forced exploration and the UCB phase only contributes regret constant in $T$. More formally, we show in \Cref{lem:terma_subexponential,lem:controlling_termb_subexponential} constant upper bounds on terms A and B respectively. 
    
    For each suboptimal arm $i$, fix $\epsilon=\epsilon_i\in(0,\Delta_i)$ (e.g.\ $\epsilon_i=\Delta_i/2$), and let $A_i$ and $B_i$ denote the right-hand sides of \Cref{eq:terma_expression} and of the second inequality in \Cref{eq:termb_expression}, with all constants computed for arm $i$. Using $\lceil g(T)\rceil\le g(T)+1$, we get $\expec{N_i(T+1)}\le g(T)+1+A_i+B_i$. Multiplying by $\Delta_i$ and summing over suboptimal arms proves the theorem with
\begin{equation}\label{eq:c_p_exact}
C_{\underline{P}}=\sum_{i:\Delta_i>0}\Delta_i\left(1+A_i+B_i\right),
\end{equation}
where
\begin{align*}
A_i&=\frac{1}{1 - e^{-C_\eta}} \Bigl( t_1 - 1 + \frac{\pi^2}{6} \Bigr) + \frac{1}{1 - e^{-C_L}} \Bigl( t_2 - 1 + \frac{\pi^2}{6} \Bigr) + t_0 - 1 + \frac{\pi^2}{3},\\
B_i&=\frac{1}{1 - e^{-C_\epsilon}} \Bigl( t_B - 1 + \frac{\pi^2}{6} \Bigr),
\end{align*}
with the constants specified in \Cref{lem:terma_subexponential,lem:controlling_termb_subexponential}.
\end{proofof}


\subsubsection{Term A}

\begin{lemma}[Controlling Term A for sub-exponential rewards.]\label[lemma]{lem:terma_subexponential}
    For any sub-optimal arm $i$, suppose that $P_i \in \mathcal{P}_{\nu_0, \alpha_0}$ for some $\nu_0, \alpha_0 > 0$. Then, Term A in \Cref{eq:termA} is bounded:
    \begin{equation}\label{eq:terma_expression}
        \textnormal{Term A} \le \frac{1}{1 - e^{-C_\eta}} \left( t_1 - 1 + \frac{\pi^2}{6} \right) + \frac{1}{1 - e^{-C_L}} \left( t_2 - 1 + \frac{\pi^2}{6} \right) + t_0 - 1 + \frac{\pi^2}{3},
    \end{equation}
    where $C_\eta, C_L > 0$ and $t_0, t_1, t_2 < \infty$ are explicit problem-dependent constants defined in the proof.
\end{lemma}

\begin{proof}
Let $\Delta_i -\epsilon = \eta_1>0$. For every $t\in S_U$, so that $N_i(t)\ge g(t)>0$ and all empirical quantities below are well defined, consider the `bad' events
$$
E_t = \{\hat{m}_i(t) - m(P_i) \geq \eta_1/2\}.
$$
From the intersection of the event $\{\indexel_i(t) - m(P_i)\geq \eta_1\}$ with $E_t^c$, we have:
$$\indexel_i(t) - m(P_i)\geq \eta_1, \, \hat{m}_i(t) - m(P_i) < \eta_1/2 \implies \indexel_i(t) - \hat{m}_i(t) \geq \eta_1/2.$$

Thus:
\begin{align}
    \sum_{t = 1}^T \prob{\indexel_i(t) - m(P_i)\geq \eta_1, A_t = i, t\in S_U}
    &\leq \sum_{t=1}^T \prob{E_t, t\in S_U} \label{eq:sum_first_part}\\ 
    &\quad + \sum_{t=1}^T\prob{\indexel_i(t) - \hat{m}_i(t) \geq \eta_1/2, t\in S_U}.\label{eq:sum_second_part}
\end{align}

The first part of the summation above (\Cref{eq:sum_first_part}) is bounded above by a constant using standard exponential concentration inequalities. 

For $E_t = \{\hat{m}_i(t) - m(P_i) \ge \eta_1/2\}$, applying the sub-exponential concentration inequality for sample means \citep[Proposition 2.9, Eq. 2.18]{Wainwright2019} with $C_\eta = \frac{1}{2} \min\left\{ \frac{\eta_1^2}{4\nu_0^2}, \frac{\eta_1}{2\alpha_0} \right\} > 0$ yields:
$$
\sum_{t=1}^{T} \prob{E_t, t \in S_U} \le \sum_{t=1}^{T} \sum_{s=\lceil g(t) \rceil}^{t-1} \exp\left( -s C_\eta \right) < \frac{1}{1 - e^{-C_\eta}} \sum_{t=1}^{T} e^{-C_\eta g(t)}.
$$
Because $g(t) = \omega(\log t)$, there exists a timestep $t_1 = t_1(\eta_1, \nu_0, \alpha_0)$ such that $C_\eta g(t) \ge 2 \log t$ for all $t \ge t_1$, giving $e^{-C_\eta g(t)} \le 1/t^2$. Splitting the summation at $t_1$ gives:
$$
\sum_{t=1}^{T} \prob{E_t, t \in S_U} \le \frac{1}{1 - e^{-C_\eta}} \left( t_1 - 1 + \frac{\pi^2}{6} \right) = \mathcal{O}(1).
$$

Moreover, we show that the same holds for the latter term (\cref{eq:sum_second_part}) using \Cref{lem:index_mean_difference_bound,lem:sample_variance_bound}. More precisely, we will show that:
$$\sum_{t=1}^T\prob{\indexel_i(t) - \hat{m}_i(t) \geq \eta_1/2, t\in S_U} = \mathcal{O}(1).$$
To do so, let $X_{i,s}$ be the $s^{\text{th}}$ sample drawn from arm $i$ and define the following relevant quantities for $t$ with $N_i(t)\ge1$:
\begin{itemize}
    \item $M_i(t) = \max_{s\leq N_i(t)} \vert{}X_{i,s} - m(P_i)\vert{},$
    \item $A_i(t) = \frac{1}{N_i(t)} \sum_{s\leq N_i(t)}\vert{}X_{i,s} - m(P_i)\vert{},$
    \item $X^{\max}_i(t) = \max_{s\leq N_i(t)} X_{i,s}.$
    \item $\hat\sigma^2_i(t)=\frac{1}{N_i(t)}\sum_{s\le N_i(t)}\bigl(X_{i,s}-\hat m_i(t)\bigr)^2.$
\end{itemize}
We write $M_{i,n}$, $A_{i,n}$, $X^{\max}_{i,n}$, $\hat m_{i,n}$ and $\hat\sigma^2_{i,n}$ for the same quantities computed from the first $n$ samples of arm $i$, so that $M_i(t)=M_{i,N_i(t)}$, $A_i(t)=A_{i,N_i(t)}$, and so on.
By \Cref{lem:index_mean_difference_bound},
$\indexel_i(t)-\hat m_i(t)\le\frac{2\crad}{N_i(t)}\bigl(X^{\max}_i(t)-\hat m_i(t)\bigr)+\sqrt{\frac{2\crad}{N_i(t)}\hat\sigma^2_i(t)}$.
Since $|\hat m_i(t)-m(P_i)|\le A_i(t)\le M_i(t)$, the triangle inequality gives $X^{\max}_i(t)-\hat m_i(t)\le 2M_i(t)$, and
\Cref{lem:sample_variance_bound} with $\mu=m(P_i)$ gives $\hat\sigma^2_i(t)\le M_i(t)A_i(t)$. Hence:
$$
\indexel_i(t) - \hat{m}_i(t)\leq \frac{4\crad M_i(t)}{N_i(t)} +\sqrt{\frac{2\crad \cdot M_i(t)\cdot A_i(t)}{N_i(t)}}.
$$

We will choose parameters $L, q \in \reals$ such that on the paths where $\{A_i(t) \leq L, M_i(t) \leq q N_i(t)\}$, we have:
$$\indexel_i(t) - \hat{m}_i(t)< \frac{\eta_1}{2}.$$
Thus, we have:
\begin{align}
\sum_{t=1}^T\prob{\indexel_i(t) - \hat{m}_i(t) \geq \frac{\eta_1}{2}, t\in S_U} &\leq \sum_{t=1}^T\prob{A_i(t) > L, t\in S_U}\label{eq:mean_absolute_deviation} \\
&\qquad + \sum_{t=1}^T\prob{M_i(t) > q N_i(t), t\in S_U}.\label{eq:max_absolute_deviation}
\end{align}

We want to choose $L, q$ such that:
$$4\crad q + \sqrt{2\crad qL} <\eta_1/2.$$

For any $L>0$, we can choose $q$ to satisfy the inequality as follows:
$$q=\frac12\min \left\{1,\ \frac{\eta_1}{16\crad},\ \frac{\eta_1^2}{32\crad L}\right\} > 0,$$
with the convention $x/0=+\infty$, so that $q=\tfrac12$ when $\crad=0$. (The factor $\tfrac12$ makes the inequality strict.)

To choose a good $L$, we need to define an environment dependent constant. We define  $\mu_A = \expecover{P_i}{\vert{}X - m(P_i)\vert{}}$ and set
$$L = \mu_A + 1.$$
Note that $L$ is only an analysis constant and need not be known to the algorithm. Since $X_{i,s}, s\geq 1$ are independent, sub-exponential random variables, the variables $\vert{}X_{i,s} - m(P_i)\vert{}, s\geq 1$ are also sub-exponential, so $\exists \ \nu_A, \alpha_A >0$ such that the law of $\vert{}X_{i,1} - m(P_i)\vert{}$ lies in $\mathcal{P}_{\nu_A, \alpha_A}$. (These are distinct from the parameters $(\nu_0,\alpha_0)$ of the original reward law.) Note that $L$ is strictly above $\mu_A$, which allows us to control the mean deviation terms.

Finally, we show that the terms on the right-hand side in \Cref{eq:mean_absolute_deviation,eq:max_absolute_deviation} are $\mathcal{O}(1)$:

\begin{itemize}
    \item Consider $\sum_{t=1}^T \prob{A_i(t) > L, t \in S_U}$. By choice of $L = \mu_A + 1$, we have $\delta_L = L - \mu_A = 1$. Applying the sub-exponential concentration inequality for sums \citep[Eq. 2.18]{Wainwright2019} to the i.i.d. sequence $Y_{i,s} = \vert{}X_{i,s} - m(P_i)\vert{}$ with law in $\mathcal{P}_{\nu_A,\alpha_A}$:
    $$\prob{A_{i,s} > L} = \prob{\frac{1}{s} \sum_{j=1}^s (Y_{i,j} - \mu_A) > \delta_L} \le \exp\left( - s C_L \right),$$
    where $C_L = \frac{1}{2} \min\left\{ \frac{1}{\nu_A^2}, \frac{1}{\alpha_A} \right\} > 0$. Summing over $s \ge g(t)$ for $t \in S_U$ gives:
    $$\sum_{t=1}^T \prob{A_i(t) > L, t \in S_U} \le \sum_{t=1}^T \sum_{s=\lceil g(t) \rceil}^{t-1} e^{-s C_L} < \frac{1}{1 - e^{-C_L}} \sum_{t=1}^T e^{-C_L g(t)}.$$
    Since $g(t) = \omega(\log t)$, there exists a timestep $t_2 = t_2(\nu_A, \alpha_A)$ such that $C_L g(t) \ge 2 \log t$ for all $t \ge t_2$, giving $e^{-C_L g(t)} \le 1/t^2$. Splitting the summation at $t_2$ gives:
    $$\sum_{t=1}^T \prob{A_i(t) > L, t \in S_U} \le \frac{1}{1 - e^{-C_L}} \left( t_2 - 1 + \frac{\pi^2}{6} \right) = \mathcal{O}(1).$$

    \item By \Cref{lem:max_dev_subexponential}, for any $t$:
    $$\prob{M_i(t) > q N_i(t), t \in S_U} \le 2(t-1) \exp\left( - \frac{1}{2} \min\left\{ \frac{q^2 g(t)^2}{\nu_0^2}, \frac{q g(t)}{\alpha_0} \right\} \right).$$
    
    Since $g(t) = \omega(\log t)$, there exists a finite timestep $t_0$ such that for all $t \ge t_0$: 
    $$\frac{1}{2} \min\left\{ \frac{q^2 g(t)^2}{\nu_0^2}, \frac{q g(t)}{\alpha_0} \right\} \ge 3 \log t.$$ 
    
    Thus:
    $$2(t-1) \exp\left( - \frac{1}{2} \min\left\{ \frac{q^2 g(t)^2}{\nu_0^2}, \frac{q g(t)}{\alpha_0} \right\} \right) \le \frac{2}{t^2},$$
    yielding:
    $$\sum_{t=1}^T \prob{M_i(t) > q N_i(t), t \in S_U} \le t_0 - 1 + \sum_{t=t_0}^T \frac{2}{t^2} \le t_0 - 1 + \frac{\pi^2}{3} = \mathcal{O}(1).$$
\end{itemize}

Combining the bounds for all three terms yields the explicit result in the lemma statement.
\end{proof}

\subsubsection{Term B}
\begin{lemma}[Controlling Term B for sub-exponential rewards.]\label[lemma]{lem:controlling_termb_subexponential}
    Suppose the optimal arm is such that $P_1 \in \mathcal{P}_{\nu_1, \alpha_1}$ for some $\nu_1, \alpha_1 > 0$. Then, Term B in \Cref{eq:termB} is bounded above:
\begin{equation}\label{eq:termb_expression}
        \textnormal{Term B} < \frac{1}{1 - e^{-C_\epsilon}} \sum_{t=1}^{T} e^{-C_\epsilon g(t)} \leq \frac{1}{1 - e^{-C_\epsilon}} \left( t_B - 1 + \frac{\pi^2}{6} \right),
\end{equation}
    where $C_\epsilon = \frac{1}{2} \min\left\{ \frac{\epsilon^2}{\nu_1^2}, \frac{\epsilon}{\alpha_1} \right\}$ and $t_B = \inf\{t: \forall t' \geq t, C_\epsilon g(t') \ge 2 \log t'\}.$
\end{lemma}

\begin{proof}
    Recall that for any time $t \in S_U$, $N_1(t) \ge g(t)$. Taking a union bound over all possible values of $N_1(t)$:
      $$
      \text{Term B} = \sum_{t=1}^{T} \mathbb{P}\left[ \hat{m}_1(t) < m(P_1) - \epsilon, t \in S_U \right] \le \sum_{t=1}^{T} \sum_{s=\lceil g(t) \rceil}^{t-1} \mathbb{P}\left( \hat{m}_{1,s} < m(P_1) - \epsilon \right).
      $$
Applying the sub-exponential tail inequality \citep[Eq. 2.18]{Wainwright2019} to $s$ i.i.d. draws yields:
  $$
  \mathbb{P}\left( \hat{m}_{1,s} < m(P_1) - \epsilon \right) = \mathbb{P}\left( \frac{1}{s} \sum_{j=1}^{s} \big(m(P_1) - X_{1,j}\big) > \epsilon \right) \le \exp(-s C_\epsilon),
  $$
where $C_\epsilon = \frac{1}{2} \min\left\{ \frac{\epsilon^2}{\nu_1^2}, \frac{\epsilon}{\alpha_1} \right\} > 0$. Bounding the inner geometric series gives:
  $$\sum_{s=\lceil g(t) \rceil}^{t-1} e^{-s C_\epsilon} < \sum_{s=\lceil g(t) \rceil}^{\infty} e^{-s C_\epsilon} = \frac{e^{-C_\epsilon \lceil g(t) \rceil}}{1 - e^{-C_\epsilon}} \le \frac{e^{-C_\epsilon g(t)}}{1 - e^{-C_\epsilon}}.$$
  
\noindent Substituting this back into the outer summation:
  
  $$\text{Term B} < \frac{1}{1 - e^{-C_\epsilon}} \sum_{t=1}^{T} e^{-C_\epsilon g(t)}.$$
    
\noindent Since $g(t) = \omega(\log t)$, there exists a finite timestep $t_B = t_B(\epsilon, \nu_1, \alpha_1)$ such that for all $t \ge t_B$:
    $$
    C_\epsilon g(t) \ge 2 \log t \implies e^{-C_\epsilon g(t)} \le e^{-2 \log t} = \frac{1}{t^2}.
    $$
\noindent Splitting the outer summation at $t_B$:
  $$
  \sum_{t=1}^{T} e^{-C_\epsilon g(t)} = \sum_{t=1}^{\min(T,t_B - 1)} e^{-C_\epsilon g(t)} + \sum_{t=t_B}^{T} e^{-C_\epsilon g(t)} \le (t_B - 1) + \sum_{t=t_B}^{T} \frac{1}{t^2} < t_B - 1 + \frac{\pi^2}{6}.
  $$
\noindent Thus, we obtain:
  $$\text{Term B} \le \frac{1}{1 - e^{-C_\epsilon}} \left( t_B - 1 + \frac{\pi^2}{6} \right) = \mathcal{O}(1).$$

\end{proof}
\subsubsection{Proofs of \Cref{lem:logel_lower_bound,lem:index_mean_difference_bound}}\label{app:main_text_lemma_proofs}

\begin{proofof}{\Cref{lem:logel_lower_bound}. }
    Abbreviate $\hat m:=m(\hat P_n)$ and $\hat\sigma^2:=\sigma^2(\hat P_n)$. Let $Y_s = m - X_s$, $\bar{Y} = m - \hat m > 0$, and $b = X_{\max} - m > 0$, so $Y_s \geq -b$ for all $s \in \{1, \dots, n\}$. For any $\lambda > 0$ satisfying $\lambda b < 1$, setting $c' = \lambda b \in (0,1)$ guarantees $\lambda Y_s \geq -c' > -1$. This implies that $\lambda$ is feasible for the dual problem given in \Cref{eq:dual_logel_defn}. Applying the logarithmic inequality $\log(1+x) \geq x - \frac{x^2}{2(1-c')}$ for $x \geq -c'$ (see \Cref{lem:exponential_inequality_empirical} for a proof), we obtain:
    $$
    \logel(\hat{P}_n, m) \geq \frac{1}{n}\sum_{s=1}^n \log(1 + \lambda Y_s) \geq \lambda \bar{Y} - \frac{\lambda^2 \hat{v}_n}{2(1 - \lambda b)},
    $$
    where $\hat{v}_n = \frac{1}{n}\sum_{s=1}^n Y_s^2 = \frac{1}{n}\sum_{s=1}^n (m - X_s)^2 = \hat\sigma^2 + (m - \hat m)^2$ is the sample second moment about $m$.

    Setting $\lambda = \frac{\bar{Y}}{\hat{v}_n + b\bar{Y}}$ (which guarantees $\lambda b = \frac{b\bar{Y}}{\hat{v}_n + b\bar{Y}} < 1$) yields:
    \begin{align*}
        \logel(\hat{P}_n, m) &\geq \frac{\bar{Y}^2}{\hat{v}_n + b\bar{Y}} - \frac{\bar{Y}^2 \hat{v}_n}{2(\hat{v}_n + b\bar{Y})^2\left(\frac{\hat{v}_n}{\hat{v}_n + b\bar{Y}}\right)} \\
        &= \frac{\bar{Y}^2}{2(\hat{v}_n + b\bar{Y})}.
    \end{align*}
    Substituting $\hat{v}_n = \hat\sigma^2 + (m - \hat m)^2$, $\bar{Y} = m - \hat m$, and $b = X_{\max} - m$, the denominator expands as:
    $$\hat{v}_n + b\bar{Y} = \hat\sigma^2 + (m - \hat m)^2 + (X_{\max} - m)(m - \hat m) = \hat\sigma^2 + (m - \hat m)(X_{\max} - \hat m).$$
    
    Plugging this back yields \Cref{eq:logel_lb_on_feasiblemeans}.
\end{proofof}

\begin{proofof}{\Cref{lem:index_mean_difference_bound}. }
Fix arm $i$ and a round $t$ with $n=N_i(t)\ge1$. Throughout this proof we abbreviate
$\hat P=\hat P_i(t)$, $\hat m=\hat m_i(t)$, $\hat\sigma^2=\hat\sigma^2_i(t)$, $X_{\max}=X^{\max}_i(t)$ and
$U=\indexel_i(t)$. By \Cref{alg:el_ucb},
$U=\sup\{m\in\reals:\ n\,\logel(\hat P,m)\le\crad\}$.

If all $n$ observations coincide, then $X_{\max}=\hat m$ and $\hat\sigma^2=0$, and by Definition~\ref{def:logel}
\[
\logel(\hat P,m)=\begin{cases}0, & m=\hat m,\\ +\infty, & \text{otherwise.}\end{cases}
\]
Thus $U=\hat m$, and the claim holds with equality. Similarly, if $\crad=0$, the feasible set is $\{\hat m\}$, so $U=\hat m$; assume $\crad>0$ from now on. If the observations are not all equal, there are at least two distinct observations, so
$X_{\max}>\hat m$ and $\hat\sigma^2>0$. Since $\logel(\hat P,\hat m)=0\le\crad/n$, the mean $m=\hat m$ is always feasible,
which gives $U\ge\hat m$, i.e.\ $0\le U-\hat m$. Moreover, $\logel(\hat P,m)=+\infty$ for all $m\ge X_{\max}$, so
every feasible $m>\hat m$ lies in $(\hat m,X_{\max})$. For such $m$, \Cref{lem:logel_lower_bound}, applied to the
$n$ observed rewards of arm $i$, gives
\[
\frac{n(m-\hat m)^2}{2\bigl(\hat\sigma^2+(m-\hat m)(X_{\max}-\hat m)\bigr)}\ \le\ n\,\logel(\hat P,m)\ \le\ \crad,
\]
which implies
\begin{equation}\label{eq:implied_condn_on_feasible_means}
(m-\hat m)^2\ \le\ \frac{2\crad}{n}\Bigl(\hat\sigma^2+(m-\hat m)(X_{\max}-\hat m)\Bigr).
\end{equation}
By \Cref{lem:set_of_solns_to_relaxation} (with $m(\hat P_n),\sigma^2(\hat P_n)$ abbreviated as $\hat m,\hat\sigma^2$), every $m>\hat m$ satisfying
\eqref{eq:implied_condn_on_feasible_means} obeys
\[
m-\hat m\ \le\ \frac{\crad}{n}(X_{\max}-\hat m)+\sqrt{\Bigl(\frac{\crad}{n}(X_{\max}-\hat m)\Bigr)^2+\frac{2\crad}{n}\hat\sigma^2}.
\]
This bound is nonnegative, so it also covers the feasible point $m=\hat m$. Taking the supremum over feasible $m$ and using
$\sqrt{a^2+b}\le a+\sqrt b$ for $a,b\ge0$, we conclude
\[
U-\hat m\ \le\ \frac{2\crad}{n}(X_{\max}-\hat m)+\sqrt{\frac{2\crad}{n}\hat\sigma^2}.
\]
\end{proofof}
\subsubsection{Supporting results}\label{app:additional_lemmas_expected}

\begin{lemma}[Strong consistency of the index]\label[lemma]{lem:index_consistency}
Fix an arm $j$ and $\crad\ge0$, and suppose that $P_j$ has a finite mean $\mu=m(P_j)$. For $n\ge1$, let
$$
\indexel_{j,n}=\sup\bigl\{m\in\reals:\ n\,\logel(\hat P_{j,n},m)\le\crad\bigr\}
$$
denote the index computed from the first $n$ samples of arm $j$. Then $\indexel_{j,n}\to\mu$ almost surely as $n\to\infty$. Consequently, $\indexel_j(t)=\indexel_{j,N_j(t)}\to\mu$ almost surely on the event $\{N_j(t)\to\infty\}$. \label[lemma]{lem:consistency_of_index}
\end{lemma}

\begin{proof}
Throughout, let
$$
M_n=\max_{s\le n}|X_{j,s}-\mu|
\qquad\text{and}\qquad
A_n=\frac1n\sum_{s\le n}|X_{j,s}-\mu|.
$$

We show upper bounds on the quantities in \Cref{lem:index_mean_difference_bound} below:
\begin{itemize}
    \item Since $|\hat m_{j,n}-\mu|\le A_n\le M_n$, the triangle inequality gives $X^{\max}_{j,n}-\hat m_{j,n}\le2M_n$.
    \item \Cref{lem:sample_variance_bound}, applied with this $\mu$, gives $\hat\sigma^2_{j,n}\le M_nA_n$.
\end{itemize}
Substituting both bounds into \Cref{lem:index_mean_difference_bound} yields
\begin{equation}\label{eq:index_consistency_sandwich}
0\ \le\ \indexel_{j,n}-\hat m_{j,n}\ \le\ 4\crad\,\frac{M_n}{n}+\sqrt{2\crad\,\frac{M_n}{n}\,A_n}.
\end{equation}
By the strong law of large numbers, $\hat m_{j,n}\to\mu$ and $A_n\to\expecover{P_j}{|X-\mu|}<\infty$ almost surely. Hence the claim follows from \eqref{eq:index_consistency_sandwich} once we show that $M_n/n\to0$ almost surely. To show this, we will demonstrate that each observation is eventually small compared to its index. Let $Y=|X_{j,1}-\mu|$; then $\E Y<\infty$ because $P_j$ has a finite mean. Fix $\epsilon>0$. The map $x\mapsto\mathbb{P}(Y>\epsilon x)$ is non-increasing, so $\mathbb{P}(Y>\epsilon n)\le\int_{n-1}^{n}\mathbb{P}(Y>\epsilon x)\,dx$ for every $n\ge1$. Since the $X_{j,n}$ are identically distributed, summing over $n$ gives
$$
\sum_{n\ge1}\mathbb{P}\bigl(|X_{j,n}-\mu|>\epsilon n\bigr)
=\sum_{n\ge1}\mathbb{P}(Y>\epsilon n)
\ \le\ \int_0^\infty\mathbb{P}(Y>\epsilon x)\,dx
\ =\ \frac{\E Y}{\epsilon}\ <\ \infty,
$$
where the equality is the identity $\E Z=\int_0^\infty\mathbb{P}(Z>x)\,dx$ for $Z=Y/\epsilon\ge0$.

By the Borel-Cantelli lemma, almost surely only finitely many of the events $\{|X_{j,n}-\mu|>\epsilon n\}$ occur. That is, there is an almost surely finite random index $N_\epsilon$ such that
\begin{equation}\label{eq:index_consistency_late_terms}
|X_{j,s}-\mu|\ \le\ \epsilon s\qquad\text{for all } s\ge N_\epsilon .
\end{equation}

Fix a sample path on which \eqref{eq:index_consistency_late_terms} holds, and let $n\ge N_\epsilon$. We split the maximum defining $M_n$ into early and late terms.
\begin{itemize}
    \item The early terms $s<N_\epsilon$ are finitely many, so their maximum $M_{N_\epsilon-1}$ is a finite number that does not depend on $n$ (with $M_0:=0$).
    \item Each late term $N_\epsilon\le s\le n$ satisfies $|X_{j,s}-\mu|\le\epsilon s\le\epsilon n$, by \eqref{eq:index_consistency_late_terms}.
\end{itemize}
Therefore $M_n\le\max\{M_{N_\epsilon-1},\,\epsilon n\}$, and so
$$
\frac{M_n}{n}\ \le\ \frac{M_{N_\epsilon-1}}{n}+\epsilon\qquad\text{for all } n\ge N_\epsilon .
$$
The numerator of the first term is fixed on this path, so $\limsup_n M_n/n\le\epsilon$.

Since this holds for any $\epsilon>0$, it holds for $\epsilon=1/k$ for each $k\in\naturals$. The intersection of these countably many probability-one events has probability one, and on it $\limsup_nM_n/n\le1/k$ for every $k$. Hence $M_n/n\to0$ almost surely.

Hence $\indexel_{j,n}-\hat m_{j,n}\to0$, and combined with $\hat m_{j,n}\to\mu$ this proves the first claim. The second claim follows because $\indexel_j(t)=\indexel_{j,N_j(t)}$ at every round $t$ with $N_j(t)\ge1$.
\end{proof}

\begin{lemma}[Characterizing solutions to \Cref{eq:implied_condn_on_feasible_means}]\label[lemma]{lem:set_of_solns_to_relaxation}
    Let $\mathcal{S}$ be the set of solutions to the inequalities:
    $$
    x^2 \leq \frac{2\crad}{n}\left( \sigma^2(\hat{P}_n) + x(X_{\max} -m(\hat{P}_n)) \right), \quad  x>0,
    $$
    where $\crad>0, n\in \naturals, \sigma^2(\hat{P}_n) > 0, X_{\max} >m(\hat{P}_n)$. 
    Then, the set $\mathcal{S}$ is equal to
    $$
    \left\{ x \in \mathbb{R} \;\middle\vert{}\; 0 < x \le \frac{\crad}{n}(X_{\max} - m(\hat{P}_n)) + \sqrt{\left(\frac{\crad}{n}(X_{\max} - m(\hat{P}_n))\right)^2 + \frac{2\crad}{n}\sigma^2(\hat{P}_n)} \right\}.
    $$
\end{lemma}
\begin{proof}
Consider the quadratic inequality: 
$$
x^2 - Ax - B \le 0,
$$
where the parameters are:
$$A = \frac{2\crad}{n}(X_{\max} - m(\hat{P}_n)), \quad B = \frac{2\crad}{n}\sigma^2(\hat{P}_n).$$
To find the critical points, we find the solutions to $x^2 - Ax - B = 0$ using the quadratic formula:
$$r_1 = \frac{A - \sqrt{A^2 + 4B}}{2}, \quad r_2 = \frac{A + \sqrt{A^2 + 4B}}{2}.$$
Because $\sigma^2(\hat P_n) > 0$ and $\crad, n > 0$, it follows that $B > 0$. This implies $\sqrt{A^2 + 4B} > \vert{}A\vert{}$, guaranteeing:
$$r_1 < 0 < r_2.$$
Finally, we intersect the feasible set with the constraint $x>0$:
$$x \in [r_1, r_2] \cap (0, \infty) \implies x \in (0, r_2].$$
Expanding $A$ and $B$ in the expression for $r_2$ yields:
$$r_2 = \frac{\crad}{n}(X_{\max} - m(\hat{P}_n)) + \sqrt{\left(\frac{\crad}{n}(X_{\max} - m(\hat{P}_n))\right)^2 + \frac{2\crad}{n}\sigma^2(\hat{P}_n)}.$$
\end{proof}

\begin{lemma}[Bounding the sample variance]\label[lemma]{lem:sample_variance_bound}
    Let $X_1, X_2, \dots X_n\in\reals$ have empirical distribution $\hat P_n$, so that their sample mean and variance are $m(\hat P_n)$ and $\sigma^2(\hat P_n)$, and let $\mu\in\reals$. Then, we have:
    $$
    \sigma^2(\hat P_n) \leq M_n\cdot A_n,
    $$
    where $M_n = \max_{s}|X_s - \mu|$ and $A_n = \frac{1}{n} \sum_{s=1}^n |X_s - \mu|.$ 
\end{lemma}

\begin{proof}
    Let $Y_s = |X_s-\mu|$ and observe:
    \begin{align*}
        \frac{1}{n} \sum_{s\in[n]} Y_s^2 &= \frac{1}{n}\sum_{s\in[n]} (X_s - \mu)^2 \\
        &= \frac{1}{n} \sum_{s\in[n]} \left[ (X_s-m(\hat P_n))^2 + (\mu-m(\hat P_n))^2 -2(X_s-m(\hat P_n))(\mu-m(\hat P_n))\right] \\
        &= \sigma^2(\hat P_n) + (\mu-m(\hat P_n))^2.
    \end{align*}
    Thus, we have $\sigma^2(\hat P_n) \leq \frac{1}{n} \sum_{s\in[n]} Y_s^2 \leq (\max_{s} Y_s)\bigl(\frac{1}{n}\sum_{s\in[n]} Y_s\bigr),$ where the latter inequality follows from the non-negativity of $Y_s$.
\end{proof}

\begin{lemma}[Maximum deviation of sub-exponential samples]
\label[lemma]{lem:max_dev_subexponential}
    Suppose $P_i \in \mathcal{P}_{\nu_0,\alpha_0}$, let $q>0$, and for $n\ge1$ let $M_{i,n} = \max_{s\leq n} \vert{}X_{i, s} - m(P_i)\vert{}$, so that $M_i(t)=M_{i,N_i(t)}$.
    \begin{enumerate}
        \item[(i)] For any $t$ (recall that $N_i(t) \ge g(t)$ whenever $t \in S_U$):
    $$
    \mathbb{P}\left[ M_i(t) > q N_i(t), t \in S_U \right] \le 2(t-1) \exp\left( - \frac{1}{2} \min\left\{ \frac{q^2 g(t)^2}{\nu_0^2}, \frac{q g(t)}{\alpha_0} \right\} \right).
    $$
        \item[(ii)] For any $n\ge1$:
    $$
    \mathbb{P}\left[ M_{i,n} > q n \right] \le 2n \exp\left( - \frac{1}{2} \min\left\{ \frac{q^2 n^2}{\nu_0^2}, \frac{q n}{\alpha_0} \right\} \right)\le2n\,e^{-c_qn},\ c_q = \frac{1}{2}\min\left\{\frac{q^2}{\nu_0^2}, \frac{q}{\alpha_0}\right\}.
    $$
    \end{enumerate}
\end{lemma}
 
\begin{proof}
    (i) Since $N_i(t) \le t-1$ at any time, and $N_i(t) \ge g(t)$ for $t \in S_U$, on the event $\{M_i(t) > qN_i(t), t\in S_U\}$ we have $\max_{s \le t-1} \vert{}X_{i,s} - m(P_i)\vert{} \ge M_i(t) > q N_i(t) \ge q\, g(t)$. Hence
    $$\{M_i(t) > qN_i(t),\, t\in S_U\} \subseteq \Big\{\max_{s\le t-1} \vert{}X_{i,s} - m(P_i)\vert{} > q\, g(t)\Big\}.$$
    The right-hand side concerns the unconditional i.i.d.\ stream, so no conditioning on the (reward-dependent) count $N_i(t)$ is needed. A union bound over $s\in[t-1]$ together with the two-sided sub-exponential tail bound (cf. \citet[Proposition 2.9]{Wainwright2019}) applied to $X_{i,s}-m(P_i)$ with $P_i \in \mathcal{P}_{\nu_0,\alpha_0}$ gives
    $$\mathbb{P}\Big[\max_{s \le t-1} \vert{}X_{i,s} - m(P_i)\vert{} > q\, g(t)\Big] \le 2(t-1) \exp\left( - \frac{1}{2} \min\left\{ \frac{q^2 g(t)^2}{\nu_0^2}, \frac{q g(t)}{\alpha_0} \right\} \right).$$
    (ii) The same union bound over $s\in[n]$, at level $qn$, gives the first inequality; the second uses $q^2n^2\ge q^2n$ for $n\ge1$.
\end{proof}

\subsection{Regret Tails}\label{app:tails}
In this section, we prove the results appearing in \Cref{sec:tails}. First, we prove \Cref{thm:tail}, which gives a finite time tail regret bound for \texttt{EL-F}. Following this, we analyse the tails of our algorithm by showing a lower bound (\Cref{thm:simchi_lb}) and the optimality of our algorithm on Gaussian instances (\Cref{cor:gaussian}).

\subsubsection{Upper bound on regret tails}\label{app:regret_tail_upper_bound}

\begin{proofof}{\Cref{thm:tail}. }
Let $x > \Delta\lceil g(T)\rceil$. Let $N_i^U(T) = \sum_{t=1}^{T} \indicator{A_t=i, t\in S_U}$ denote the number of pulls of arm $i$ in the UCB phase. Splitting the pulls into forced and UCB pulls, 
$$R(T) = R_{\mathrm{forced}}(T) + \sum_{i: \Delta_i > 0} \Delta_i N_i^U(T),$$ 
where $R_{\mathrm{forced}}(T) \le \Delta \lceil g(T)\rceil$. Decomposing the total regret yields
\begin{align}
    \prob{R(T) \geq x} &\le \prob{\sum_{i: \Delta_i > 0} N_i^U(T)\Delta_i \geq x - \Delta\lceil g(T)\rceil} \nonumber\\
    &\leq \prob{\bigcup_{i: \Delta_i > 0}\left\{N_i^U(T) \Delta_i \geq \frac{(x-\Delta\lceil g(T)\rceil)\Delta_i}{\Delta}\right\}}\nonumber \\
    &\leq \sum_{i: \Delta_i > 0} \prob{N_i^U(T) \geq \frac{x-\Delta\lceil g(T)\rceil}{\Delta}}.\label{eq:tail_bound_suboptimalpulls}
\end{align}
Let 
$$
k_x = \left\lceil\frac{x-\Delta\lceil g(T)\rceil}{\Delta}\right\rceil \ge 1.
$$ 
We will control the summand in \Cref{eq:tail_bound_suboptimalpulls} for any sub-optimal arm $i$. Observe that since $N_i^U(T)$ is an integer, the summand in \Cref{eq:tail_bound_suboptimalpulls} equals $\prob{N_i^U(T)\ge k_x}$. 


Assume for now $k_x\ge2$, in which case $g(k_x)\le k_x-1$ by assumption. Observe that $\{N_i^U(T)\ge k_x\}$ implies that there have been at least $k_x$ instances in the UCB phase such that the index of the sub-optimal arm is at least that of the optimal arm. However, the sample counts at these instances are random. We will proceed by showing that this event is contained inside a union over possible sample counts, i.e. 
\begin{equation}\label{eq:containment_tail_pulls}
\{N_i^U(T)\ge k_x\}\ \subseteq\ \Big\{\exists\, n\ge k_x-1,\ s\ge\lceil g(k_x)\rceil:\ \indexel_{i,n}\ge\indexel_{1,s}\Big\},
\end{equation}
where $\indexel_{j,n}$ denotes the index of arm $j$ computed from its first $n$ samples.

To show the set inclusion in \Cref{eq:containment_tail_pulls}, let $\tau$ be the round of the $k_x$-th pull of arm $i$ in the UCB phase. Then $\tau\in S_U$, $A_\tau=i$ and $\tau\ge k_x$. Moreover, we have the following observations:
\begin{itemize}
    \item $N_i(\tau)\ge k_x-1$, by definition of $\tau$;
    \item $N_1(\tau)\ge g(\tau)\ge g(k_x)$, since $\tau\in S_U$ and $g$ is non-decreasing;
    \item $\indexel_i(\tau)\ge\indexel_1(\tau)$, since arm $i$ maximizes the index at round $\tau$.
\end{itemize}
These observations and a union over all possible values of $N_i(\tau)$ and $N_1(\tau)$ show the desired set inclusion in \Cref{eq:containment_tail_pulls}. 

Next, we would like to analyze the event on the right hand side of \Cref{eq:containment_tail_pulls} in a manner similar to that in the proof of \Cref{thm:expected_regret}. We will demonstrate that this event is itself contained inside the union of four other events (defined in \cref{eq:event_containment_one,eq:event_containment_two}). To set up these events, we demonstrate a key inequality that will be used to control the index deviation of the sub-optimal arm.

Let $A_{i, n} = \frac{1}{n}\sum_{j=1}^n \vert{}X_{i, j} - m(P_i)\vert{}$ and $M_{i, n} = \max_{1 \le j \le n} \vert{}X_{i, j} - m(P_i)\vert{}$. We claim
\begin{equation}\label{eq:index_upper_bound_tailproof}
\indexel_{i, n} - \hat{m}_{i, n} \le \frac{4\crad M_{i, n}}{n} + \sqrt{\frac{2\crad M_{i, n} A_{i, n}}{n}},
\end{equation}
where $\hat m_{i,n}$ is the empirical mean from the first $n$ samples of arm $i$. The claim follows from applying \Cref{lem:index_mean_difference_bound}, an upper bound on the sample variance term (\Cref{lem:sample_variance_bound}) and finally, applying the triangle inequality to the range term as $\max_{j\leq n}X_{i,j}-\hat m_{i,n}\le 2M_{i,n}$. We note that since $\indexel_{j,n}$ equals $\indexel_j(t)$ at any round $t$ with $N_j(t)=n$, \Cref{lem:index_mean_difference_bound} applies to $\indexel_{j,n}$ for every $n\ge1$. 

To define the events, fix $\epsilon \in (0, \Delta_i)$, and define $\eta_1 = \Delta_i - \epsilon > 0$, $L = \expecover{P_i}{|X-m(P_i)|} + 1$ and $q = \frac12\min\left\{1,\ \frac{\eta_1}{16\crad},\ \frac{\eta_1^2}{32\crad L}\right\}$ (with $x/0=+\infty$). Then
\begin{align}
\mathcal{E}_{1, s} &= \left\{ \hat{m}_{1, s} < m(P_1) - \epsilon \right\}, \quad
&&\mathcal{E}_{2, n} = \left\{ \hat{m}_{i, n} - m(P_i) \ge \frac{\eta_1}{2} \right\}, \label{eq:event_containment_one}\\
\mathcal{E}_{3, n} &= \left\{ A_{i, n} > L \right\}, \quad
&&\mathcal{E}_{4, n} = \left\{M_{i, n} > q n \right\}.\label{eq:event_containment_two}
\end{align}
We argue that
$$
\Big\{\indexel_{i,n}\ge\indexel_{1,s}\Big\} \subseteq \mathcal{E}_{1, s}\cup\mathcal{E}_{2, n} \cup\mathcal{E}_{3, n} \cup \mathcal{E}_{4, n}.
$$
To see this, observe that if none of the four events occurs, then $\indexel_{1,s}\ge\hat m_{1,s}\ge m(P_1)-\epsilon$, while, by \Cref{eq:index_upper_bound_tailproof} and the choice of $L$ and $q$, $\indexel_{i,n}-m(P_i)=(\indexel_{i,n}-\hat m_{i,n})+(\hat m_{i,n}-m(P_i))<\eta_1$, i.e.\ $\indexel_{i,n}<m(P_i)+\eta_1=m(P_1)-\epsilon\le\indexel_{1,s}$. Hence $\indexel_{i,n}\ge\indexel_{1,s}$ does not hold. Equivalently, whenever $\indexel_{i,n}\ge\indexel_{1,s}$, at least one of the four events occurs.
 
\noindent Finally, a union over sample counts $s \ge \lceil g(k_x)\rceil$ and $n \ge k_x - 1$ (note $k_x-1\ge1$) gives
$$
\Big\{\exists\, n\ge k_x-1,\ s\ge\lceil g(k_x)\rceil:\ \indexel_{i,n}\ge\indexel_{1,s}\Big\} \subseteq \bigcup_{s\geq \lceil g(k_x)\rceil} \bigcup_{n\geq k_x-1} \left(\mathcal{E}_{1, s}\cup\mathcal{E}_{2, n} \cup\mathcal{E}_{3, n} \cup \mathcal{E}_{4, n}\right),
$$
which by \Cref{eq:containment_tail_pulls} and a union bound implies
\begin{align}
    \prob{N_i^U(T) \ge k_x} &\le \sum_{s = \lceil g(k_x) \rceil}^\infty \prob{\mathcal{E}_{1, s}} + \sum_{n = k_x - 1}^\infty \left( \prob{\mathcal{E}_{2, n}} + \prob{\mathcal{E}_{3, n}} + \prob{\mathcal{E}_{4, n}} \right) \label{eq:tail_union_bound} \\
    &\leq \frac{e^{-C_\epsilon g(k_x)}}{1 - e^{-C_\epsilon}} + \frac{e^{-C_\eta (k_x - 1)}}{1 - e^{-C_\eta}} + \frac{e^{-C_L (k_x - 1)}}{1 - e^{-C_L}} +  \frac{2k_x e^{-c_q (k_x - 1)}}{(1 - e^{-c_q})^2}\label{eq:summation_exponential},
\end{align}
where 
$$C_\epsilon = \frac{1}{2}\min\left\{\frac{\epsilon^2}{\nu_1^2}, \frac{\epsilon}{\alpha_1}\right\}, C_\eta = \frac{1}{2}\min\left\{\frac{\eta_1^2}{4\nu_0^2}, \frac{\eta_1}{2\alpha_0}\right\}, C_L = \frac12\min\left\{\frac{1}{\nu_A^2}, \frac{1}{\alpha_A}\right\}$$ 
and $c_q>0$ as defined in \Cref{lem:max_dev_subexponential}(ii). The first three terms in \Cref{eq:summation_exponential} follow from concentration for sums of sub-exponential random variables with appropriate parameters such that $\textnormal{law}(X_{1,s}) \in \mathcal{P}_{\nu_1, \alpha_1}, \textnormal{law}(X_{i,n}) \in \mathcal{P}_{\nu_0, \alpha_0}$ and $\textnormal{law}(|X_{i,n} - m(P_i)|) \in \mathcal{P}_{\nu_A, \alpha_A}$. The last term in \Cref{eq:summation_exponential} follows from \Cref{lem:max_dev_subexponential} and the following inequality
$$
\sum_{n\ge k-1}n r^n=r^{k-1}\frac{(k-1)(1-r)+r}{(1-r)^2}\le\frac{k\,r^{k-1}}{(1-r)^2},
$$ 
for $r\in(0,1)$ and $k\ge2$.
 
 
 
\noindent Simplifying \Cref{eq:summation_exponential} using $g(k_x) \le k_x - 1$ for $k_x \ge 2$, we obtain the following explicit finite-time bound
\begin{align}
\prob{N_i^U(T) \ge k_x}
&\le \left( \frac{1}{1 - e^{-C_\epsilon}} + \frac{1}{1 - e^{-C_\eta}} + \frac{1}{1 - e^{-C_L}} + \frac{2}{(1 - e^{-c_q})^2} \right) k_x e^{-C_{\min,i} g(k_x)} \nonumber \\
&\le C_{0,i}\, k_x e^{-C_{\min,i} g(k_x)}\label{eq:tail_upper_bound_armi},
\end{align}
where $C_{\min,i} = \min\{C_\epsilon, C_\eta, C_L, c_q\} > 0$ and 
\begin{equation}\label{eq:c0_expression}
C_{0,i} := \max\Big\{\frac{1}{1 - e^{-C_\epsilon}} + \frac{1}{1 - e^{-C_\eta}} + \frac{1}{1 - e^{-C_L}} + \frac{2}{(1 - e^{-c_q})^2},\ e^{C_{\min,i} g(1)}\Big\} < \infty
\end{equation} are instance-dependent constants depending on the arm $i$ but independent of $T, k,$ and $x$. 

Notice that since we include the factor $e^{C_{\min,i} g(1)}$ in $C_{0,i}$ in \Cref{eq:c0_expression}, if $k_x=1$, we obtain the trivial bound $\prob{N_i^U(T)\ge k_x}\le 1$, ensuring \Cref{eq:tail_upper_bound_armi} holds for all $k_x$.

Set $C_0=\max_{i:\Delta_i>0}C_{0,i}$ and $C_{\min}=\min_{i:\Delta_i>0}C_{\min,i}$ and sum over the at most $K-1$ suboptimal arms (note $C_{0,i}\le C_0$, $C_{\min,i}\ge C_{\min}$ and $g>0$) to obtain the key finite time-result:
 
\begin{equation*}
\prob{R(T) \ge x} \le \sum_{i: \Delta_i > 0} \prob{N_i^U(T) \ge k_x} \le (K-1) C_0 \, k_x \exp\left( - C_{\min}\, g\left( k_x \right) \right).
\end{equation*}
Taking the logarithm proves part (i).

\noindent For part (ii), let $g$ be slowly varying and $x_T=\omega(g(T))$. Then $x_T>\Delta\lceil g(T)\rceil$ for all sufficiently large $T$. Dividing the result of (i) by $g(x_T) > 0$ and taking the limit as $T\to \infty$ yields
\begin{align*}
    \limsup_{T\to \infty}\frac{\log \mathbb{P}(R(T) \ge x_T)}{g(x_T)} &\le \limsup_{T\to \infty}\left(\frac{\log ((K-1) C_0)}{g(x_T)} + \frac{\log k_{x_T}}{g(x_T)} - C_{\min} \frac{g\left( k_{x_T} \right)}{g(x_T)}\right) \\
    &\leq - C_{\min}\liminf_{T\to \infty}\frac{g(k_{x_T})}{g(x_T)},
\end{align*}
where the first term in the penultimate inequality trivially goes to zero, while the second term going to zero follows from the observation that since $k_{x_T}\le x_T/\Delta+1$, we have $\log k_{x_T} = \mathcal{O}(\log x_T)$, and $\log x_T=o(g(x_T))$ because $g(t) = \omega(\log t).$ Finally, from monotonicity, we obtain
$$
g\left( k_{x_T} \right) \ge g\left( \frac{x_T - \Delta \lceil g(T)\rceil}{\Delta} \right),
$$
which yields the desired result upon using $x_T = \omega(g(T))$ and slow variation.
\end{proofof}

\subsubsection{Lower bound}\label{app:lower_bound_tails}

\begin{proofof}{\Cref{thm:simchi_lb}. }
To show the lower bound, we proceed using a standard change of measure approach. Using the data processing inequality, we will lower bound the $\kl$ divergence between the distributions of the history induced by the policy under the two instances. 

Let $\mathbb{P}$ and $\mathbb{Q}$ denote the laws of the history $H_T=(A_1,X_1,\dots,A_T,X_T)$ induced by $\pi$ on $\underline{P}$ and $\underline{Q}$, respectively, so that $\mathbb{P}=\mathbb{P}_{\underline{P}}$. Let
$$
\bar N_T := \expecover{\underline{Q}}{N_1(T+1)}
$$
denote the expected number of pulls of arm $1$ under $\mathbb{Q}$. In particular, $\Delta_Q\bar N_T=\expecover{\underline{Q}}{R(T)}=o(T)$.

Define the event
$$
A := \{N_2(T+1) > x_T/\Delta\},
$$
which coincides with $\{R(T)>x_T\}$ under $\underline{P}$; hence $\mathbb{P}_{\underline{P}}[R(T)>x_T]=\mathbb{P}(A)$.

Since arm 2 is optimal under $\underline{Q}$, we must have that $\mathbb{Q}(A)$ is large. To make this precise, observe that since $\limsup_T x_T/T<\Delta$, there exists $b\in(0,1]$ such that $x_T/\Delta\le(1-b)T$ for all sufficiently large $T$. On $A^c$ we then have $N_2(T+1)\le(1-b)T$, and since $N_1(T+1)+N_2(T+1)=T$, we also have $N_1(T+1)\ge bT$. By Markov's inequality,
\begin{equation}\label{eq:markovs_qac}
    \mathbb{Q}(A^c)\ \le\ \mathbb{Q}\big(N_1(T+1)\ge bT\big)\ \le\ \frac{\bar N_T}{bT}\ \longrightarrow\ 0,
\end{equation}
because $\bar N_T=o(T)$. Hence $\mathbb{Q}(A)\to1$, and in particular $\mathbb{Q}(A)>0$ for all sufficiently large $T$, which we assume from now on.

Next, we apply the change of measure machinery from \citet[Lemma 1]{kaufmann2016complexity}, which essentially corresponds to applying the chain rule for $\kl$ divergences followed by the data processing inequality. In particular, note that the policy is same under both instances and the instances only differ in arm 1, which implies
$$
\kl(\mathbb{Q},\mathbb{P})
=\bar N_T\,\kl\big(N(\mu_{1,Q},\sigma^2),N(\mu_1,\sigma^2)\big)
=\bar N_T\,\frac{(\mu_1-\mu_{1,Q})^2}{2\sigma^2}
=\bar N_T\,\frac{(\Delta+\Delta_Q)^2}{2\sigma^2},
$$
since $\mu_1-\mu_{1,Q}=(\mu_1-\mu_2)+(\mu_2-\mu_{1,Q})=\Delta+\Delta_Q$. By the data-processing inequality,
\begin{equation}\label{eq:binary_kl_dpi}
\kl(\mathbb{Q},\mathbb{P})\ \ge\ \mathrm{kl}\big(\mathbb{Q}(A),\mathbb{P}(A)\big),
\end{equation}
where $\mathrm{kl}(a,b)=a\log\frac ab+(1-a)\log\frac{1-a}{1-b}$, with the convention $\mathrm{kl}(a,0)=+\infty$ for $a>0$.

To analyze the right hand side of \Cref{eq:binary_kl_dpi}, we consider different cases for the value of $p:=\mathbb{P}(A)$:
\begin{itemize}
\item If $p\in(0,1)$, \Cref{lem:binary_kl} with $a=\mathbb{Q}(A)$ and $b=p$ yields
$$
\bar N_T\,\frac{(\Delta+\Delta_Q)^2}{2\sigma^2}\ \ge\ \mathbb{Q}(A)\log\frac1{p}-\log2.
$$
Rearranging and dividing by $\mathbb{Q}(A)>0$,
\begin{equation}\label{eq:lower_bound_on_pA}
\log\mathbb{P}(A)\ \ge\ -\frac{1}{\mathbb{Q}(A)}\left[\frac{(\Delta+\Delta_Q)^2}{2\sigma^2}\,\bar N_T+\log2\right].    
\end{equation}
\item $p=0$ is impossible: since $\mathbb{Q}(A)>0$, the right-hand side of \Cref{eq:binary_kl_dpi} would be $+\infty$, contradicting the finiteness of the left-hand side since $\bar N_T\le T$.
\item If $p=1$, then $\log p=0$ and \Cref{eq:lower_bound_on_pA} holds trivially, because the right-hand side is non-positive.

\end{itemize}
Since $\mathbb{Q}(A)\to1$ by \Cref{eq:markovs_qac}, we have $1/\mathbb{Q}(A)=1+o(1)$. As $\mathbb{P}(A)=\mathbb{P}_{\underline{P}}[R(T)>x_T]$, this gives
$$
\log\mathbb{P}_{\underline{P}}\big[R(T)>x_T\big]\ \ge\ -(1+o(1))\left[\frac{(\Delta+\Delta_Q)^2}{2\sigma^2}\,\expecover{\underline{Q}}{N_1(T+1)}+\log2\right].
$$
\end{proofof}

\subsubsection{Optimality on Gaussian instance}\label{app:optimality_tails}
 
\begin{proofof}{\Cref{cor:gaussian}. }
First, we will show an upper bound on the regret tail induced by \texttt{EL-F} on this instance. Following this, for the lower bound, we show that the algorithm satisfies the hypotheses of \Cref{thm:simchi_lb} and take $\Delta_Q\to 0$. 

We first observe that slow variation of $g$ implies that $g$ grows sublinearly, i.e., $g(T)=o(T)$. Indeed, by slow variation there is a $t_0$ such that $g(2t)\le\sqrt2\,g(t)$ for all $t\ge t_0$; iterating this together with the monotonicity of $g$ gives $g(T)\le g(t_0)\sqrt{2T/t_0}$ for all $T\ge t_0$.

\emph{Upper bound.} We follow the proof of \Cref{thm:tail} up to the point where concentration bounds are applied, and there use concentration inequalities that are tight for the Gaussian environment. Let $a=\Delta^2/(2\sigma^2)$, fix $\epsilon\in(0,\Delta)$ and $c'\in(0,c)$, and let $a_\epsilon=\epsilon^2/(2\sigma^2)$.

Define the deviation term
$$
k_T=\left\lceil\frac{cT-\Delta\lceil g(T)\rceil}{\Delta}\right\rceil .
$$
Since $g(T)=o(T)$ and $c'<c<\Delta$, for all sufficiently large $T$ we have $cT>\Delta\lceil g(T)\rceil$ and $c'T/\Delta\le k_T\le T$; in particular, $k_T\ge2$. Moreover, monotonicity gives $g(c'T/\Delta)\le g(k_T)\le g(T)$, so slow variation implies
\begin{equation}\label{eq:slow_variation_conseqeunce}
\lim_{T\to\infty}\frac{g(k_T)}{g(T)}=1.
\end{equation}
Since $\{R(T)>cT\}\subseteq\{R(T)\ge cT\}$, \Cref{eq:tail_bound_suboptimalpulls,eq:tail_union_bound} from \Cref{app:regret_tail_upper_bound} give
$$
\prob{R(T)>cT}\ \le\ \prob{N_2^U(T)\ge k_T}\ \le\ S_1(T)+S_2(T),
$$
where
$$
S_1(T)=\sum_{s\ge\lceil g(k_T)\rceil}\prob{\mathcal{E}_{1,s}},\qquad
S_2(T)=\sum_{n\ge k_T-1}\Big(\prob{\mathcal{E}_{2,n}}+\prob{\mathcal{E}_{3,n}}+\prob{\mathcal{E}_{4,n}}\Big),
$$
with the events of \Cref{eq:event_containment_one,eq:event_containment_two} defined for $i=2$ and the fixed $\epsilon$.

The term $S_1$ concerns the optimal arm. Here $\mathcal{E}_{1,s}=\{\hat m_{1,s}<\mu_1-\epsilon\}$ and $\hat m_{1,s}\sim N(\mu_1,\sigma^2/s)$, so the Gaussian tail bound gives $\prob{\mathcal{E}_{1,s}}\le e^{-a_\epsilon s}$. Summing the geometric series yields
$$
S_1(T)\ \le\ C_1\,e^{-a_\epsilon g(k_T)},\qquad C_1=\frac{1}{1-e^{-a_\epsilon}}.
$$
The term $S_2$ concerns the suboptimal arm and does not involve $g$. Since $N(\mu_2,\sigma^2)\in\mathcal{P}_{\textnormal{exp}}$, the constants $C_\eta,C_L,c_q>0$ are similarly well defined, and the same geometric bounds give
$$
S_2(T)\ \le\ \frac{e^{-C_\eta(k_T-1)}}{1-e^{-C_\eta}}+\frac{e^{-C_L(k_T-1)}}{1-e^{-C_L}}+\frac{2k_T\,e^{-c_q(k_T-1)}}{(1-e^{-c_q})^2}\ \le\ C_2\,k_T\,e^{-\rho(k_T-1)},
$$
where $\rho=\min\{C_\eta,C_L,c_q\}>0$ and $C_2=\frac{1}{1-e^{-C_\eta}}+\frac{1}{1-e^{-C_L}}+\frac{2}{(1-e^{-c_q})^2}<\infty$ do not depend on $T$. 

Again, using $c'T/\Delta\le k_T\le T$, we obtain
$$
\log S_2(T)\ \le\ \log C_2+\log T+\rho-\rho\,\frac{c'T}{\Delta},
$$
which decreases linearly in $T$, whereas $a_\epsilon g(k_T)\le a_\epsilon g(T)=o(T)$. Hence $S_2(T)\le e^{-a_\epsilon g(k_T)}$ for all sufficiently large $T$, and therefore
$$
\frac{\log\prob{R(T)>cT}}{g(T)}\ \le\ \frac{\log(C_1+1)}{g(T)}\ -\ a_\epsilon\,\frac{g(k_T)}{g(T)}
$$
for all sufficiently large $T$. As $T\to\infty$, the first term on the right-hand side tends to $0$ because $g(T)\to\infty$, and we also have $g(k_T)/g(T)\to1$ by \Cref{eq:slow_variation_conseqeunce}. Hence the $\limsup$ of the left-hand side is at most $-a_\epsilon$. Since $\epsilon\in(0,\Delta)$ was arbitrary, letting $\epsilon\uparrow\Delta$ yields
$$
\limsup_{T\to\infty}\frac{\log\prob{R(T)>cT}}{g(T)}\ \le\ -a.
$$

\emph{Lower bound.} Fix $\eta>0$ and let $\underline{Q}^\eta=(N(\mu_2-\eta,\sigma^2),N(\mu_2,\sigma^2))$. Let $$\bar N_T=\expecover{\underline{Q}^\eta}{N_1(T+1)}.$$ Since $\underline{Q}^\eta\in\mathcal{P}_{\exp}^2$, \Cref{thm:expected_regret} applies and gives
$$
\eta\,\bar N_T\ =\ \expecover{\underline{Q}^\eta}{R(T)}\ \le\ \eta\,g(T)+C_{\underline{Q}^\eta},
$$
where $C_{\underline{Q}^\eta}<\infty$ does not depend on $T$. Since $g(t) = o(t)$, this implies $\expecover{\underline{Q}^\eta}{R(T)}=o(T)$ and
$$
\bar N_T\ \le\ g(T)+C_{\underline{Q}^\eta}/\eta.
$$
Moreover, $x_T=cT$ satisfies $\limsup_T x_T/T=c<\Delta$. Hence \Cref{thm:simchi_lb} applies to \texttt{EL-F} and gives
$$
\log\prob{R(T)>cT}\ \ge\ -(1+o(1))\left[\frac{(\Delta+\eta)^2}{2\sigma^2}\Big(g(T)+\frac{C_{\underline{Q}^\eta}}{\eta}\Big)+\log 2\right].
$$
Dividing by $g(T)$ and using $g(T)\to\infty$, we obtain
$$
\liminf_{T\to\infty}\frac{\log\prob{R(T)>cT}}{g(T)}\ \ge\ -\frac{(\Delta+\eta)^2}{2\sigma^2}.
$$
Since $\eta>0$ was arbitrary, letting $\eta\downarrow0$ shows that the $\liminf$ is at least $-a$. Together with the upper bound, the limit exists and equals $-a=-\Delta^2/(2\sigma^2)$.
\end{proofof}

\subsection{Technical lemmas}\label{sec:extra_lemmas}
\begin{lemma}\label[lemma]{lem:exponential_inequality_empirical}
    For any parameter $c \in (0, 1)$ and any $x \ge -c$:
    $$
    \log(1 + x) \ge x - \frac{x^2}{2(1-c)}.
    $$ 
\end{lemma}
\begin{proof}
        Define $f(x) = \log(1+x) - x + \frac{x^2}{2(1-c)}$ on $[-c, \infty)$. Note that $f(0) = 0$. Taking the derivative:
        $$
        f'(x) = \frac{1}{1+x} - 1 + \frac{x}{1-c} = \frac{-x}{1+x} + \frac{x}{1-c} = x \left( \frac{1}{1-c} - \frac{1}{1+x} \right).$$
        For $x \in [-c, 0]$, $1+x \ge 1-c \implies \frac{1}{1+x} \le \frac{1}{1-c} \implies f'(x) \le 0$. Thus, $f(x)$ is decreasing on $[-c, 0]$, meaning $f(x) \ge f(0) = 0$.
        For $x \ge 0$, $1+x \ge 1 > 1-c \implies \frac{1}{1+x} < \frac{1}{1-c} \implies f'(x) \ge 0$. Thus, $f(x)$ is increasing on $[0, \infty)$, meaning $f(x) \ge f(0) = 0$. Hence, $f(x) \ge 0$ for all $x \ge -c$.
\end{proof}

\begin{lemma}[A binary $\kl$ inequality]\label[lemma]{lem:binary_kl}
For $a\in[0,1]$ and $b\in(0,1)$, $\mathrm{kl}(a,b):=a\log\frac ab+(1-a)\log\frac{1-a}{1-b}\ \ge\ a\log\frac1b-\log2$.
\end{lemma}
\begin{proof}
$\mathrm{kl}(a,b)=-h(a)+a\log\frac1b+(1-a)\log\frac1{1-b}$, where $h(a)=-a\log a-(1-a)\log(1-a)\le\log2$ and $\log\frac1{1-b}\ge0$.
\end{proof}

\section{Additional Experiments}\label{app:experiments}

We supplement the results in \Cref{sec:experiments} with a detailed experimental protocol and comparisons of all algorithms on six bandit instances. We report mean regret, empirical regret tails, and the frequency of trapped runs to illustrate how forced exploration and the empirical-likelihood
decision radius influences performance.


\paragraph{Experimental setup.} For each algorithm and instance, we perform $1{,}024$ independent runs of $T=10{,}000$ rounds. We use common random numbers to compare algorithms: within each run, we generate a reward sequence for each arm and use the same sequences across algorithms. Thus, the $n$-th pull of an arm returns the same reward under every algorithm, regardless of
the round in which it is selected.

The six instances comprise the four described in \Cref{tab:instances} and the two Beta instances without the scaling factor $50$. The unscaled Beta instances have rewards in $[0,1]$, with means and gaps equal to one-fiftieth of those of their counterparts on $[0,50]$.

In our implementation, \texttt{EL-F} forces exploration whenever an arm has fewer than $\max\{g(t),2\}$ samples, ensuring that every arm is sampled at least twice before the UCB phase. For the schedule $g_1$, we replace $\log\log t$ by $0$ whenever $\log t\le1$. These conventions specify the handling of initial rounds mentioned in \Cref{sec:experiments}. Both \texttt{R-UCB} and \texttt{RDS} initialize by pulling each arm once.

We report the empirical mean, median, and tail probabilities of the pseudo-regret $R(T)=\sum_a\Delta_a N_a(T+1)$. We classify a run as \emph{trapped} if the optimal arm is selected fewer than
$T/2=5{,}000$ times over the horizon. This criterion provides a finite-horizon measure of poor allocation to the optimal
arm, rather than necessarily indicating permanent trapping.


\subsection{Additional Discussion}
\label{app:all_algo_experiment}

\Cref{fig:regret_curves_all,fig:regret_tail_all} compare the mean pseudo-regret and empirical regret tails of all algorithms on the six instances. The paired Beta instances allow us to examine the effect of multiplying rewards by $50$.

\paragraph{Effect of reward scaling.}
\label{para:effect-of-reward-scaling}
The decisions of \texttt{EL-F}, \texttt{RDS}, and \texttt{FE} are invariant under positive reward scaling. Consequently, their mean and median regret on the Beta instances supported on $[0,50]$ are exactly $50$ times those on the corresponding instances supported on $[0,1]$. 

In contrast, the exploration bonus of \texttt{R-UCB}, $\sqrt{\varphi(t)/N_i(t)}$, does not scale with the rewards. Its behaviour therefore depends on the size of this bonus relative to the reward gaps. For example, on the wide-gap Beta instance with rewards in $[0,1]$, \texttt{R-UCB} achieves mean and median regret of $31.8$ with $\varphi=g_1$. After scaling the rewards by $50$, its median regret increases to $75$, while its mean regret increases to $833$ (\Cref{tab:regret-summary}).

Thus, although \texttt{R-UCB} is scale-free
in the sense of requiring no prior knowledge
of reward-scale parameters, it is not
invariant under reward rescaling.

\subsection{Ablation of the empirical-likelihood radius}\label{app:ablation}

We examine how the radius $\crad$ affects the performance of \texttt{EL-F} under each exploration schedule $g\in\{g_1,g_2\}$. We compare $\crad=0$ and $\crad=3$ used in the main experiments, and two time-dependent radii, $\crad(t)=\log\log t+3$ and $\crad(t)=\log t+3$. We use the same six instances and experimental protocol as above, including the convention for $\log\log t$ in the initial rounds. \Cref{fig:crad_ablation_curves,fig:crad_ablation_tail}
show the empirical mean regret and regret tails,
while \Cref{tab:crad-ablation,tab:trapped-all}
report the mean and median regret and the
frequency of trapped runs, respectively. 

\paragraph{Benefit of local optimism.}
A positive radius substantially improves mean regret on the instances where the greedy variant has frequent poor runs. On the scaled Beta instance with small gaps, setting $\crad=3$ reduces mean regret from $3{,}001$ to $765$ under $g_1$, while increasing median regret from $167$ to $210$. On the exponential instance, the corresponding mean decreases from $24{,}255$ to $7{,}672$, while the median increases from $450$ to $660$. The larger exploration schedule $g_2$ complements this optimism. Thus, local optimism can reduce the contribution of poor runs to mean regret, at the cost of additional exploration in typical runs.

\paragraph{Effect of local optimism on trapped runs.}
To further examine the role of local optimism,
we report the frequency of trapped runs for
\texttt{EL-F} in \Cref{tab:trapped-all}.
Under $g_1$, increasing $\crad$ from $0$ to $3$
reduces this frequency from $8.1\%$ to $1.2\%$
on the small-gap Beta instance, and from
$14.6\%$ to $3.5\%$ on the exponential instance.
Under $g_2$, the corresponding frequencies
decrease from $0.9\%$ to $0\%$ and from
$2.0\%$ to $0.4\%$, respectively.

These results suggest that a positive optimism radius helps \texttt{EL-F} avoid prolonged suboptimal allocation in finite-horizon runs, complementing the protection provided by forced exploration.

\paragraph{Time-dependent radii.}
Increasing the radius with time gives further improvements on some instances, but the gains are not uniform. With $g_1$, the radius $\log\log t+3$ reduces mean regret on the scaled Beta instance with small gaps from $765$ to $586$. On the exponential instance, however, the mean improves only modestly. With $g_2$, the same radius reduces mean regret on the exponential instance from $2{,}607$ to $2{,}389$, but slightly increases it on the Beta instance with small gaps. The more aggressive radius $\log t+3$ generally raises median regret and does not consistently improve the mean relative to $\log\log t+3$. These results support the use of a small positive constant radius as a simple choice, while showing that a growing radius can help on particular instances. Growing radii also require recomputing every arm's index in each UCB round, whereas a constant radius requires updating only the index of the pulled arm.

\begin{figure}[htbp]
    \centering
    \includegraphics[width=\textwidth]{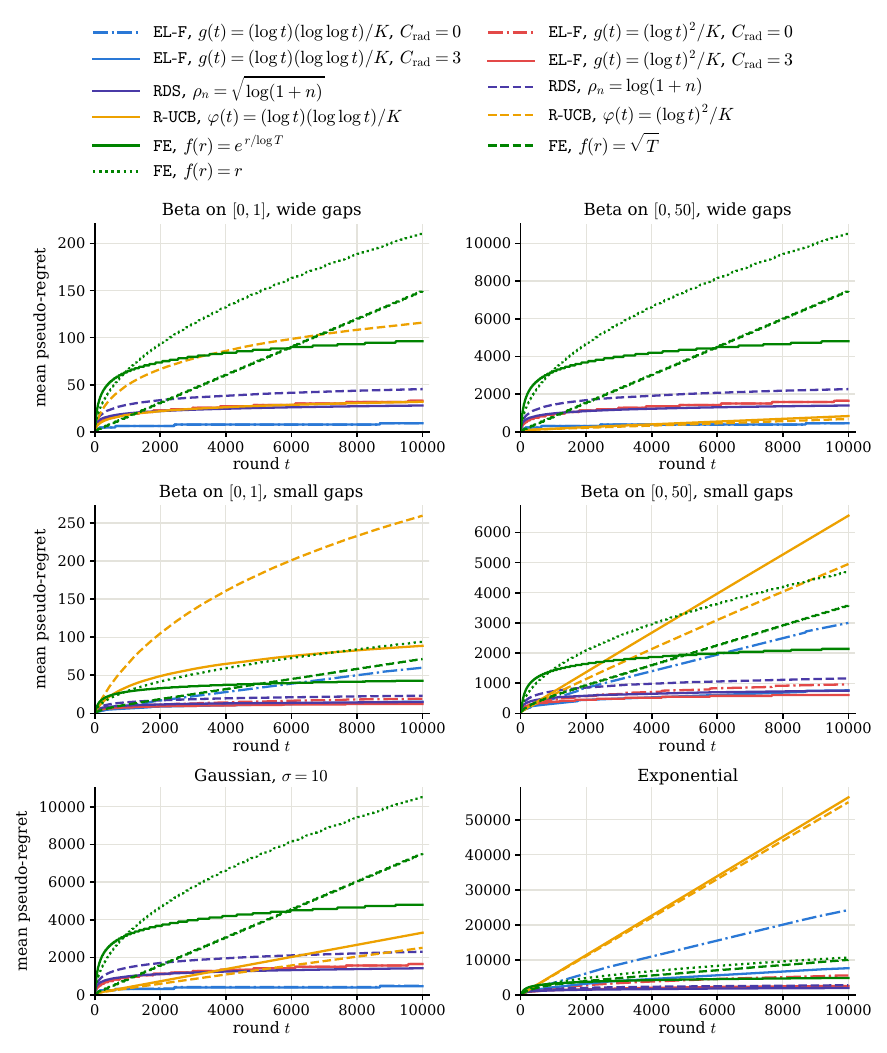}
    \caption{Mean pseudo-regret over $1{,}024$ runs on all six instances, for every algorithm.}
    \label{fig:regret_curves_all}
\end{figure}

\begin{figure}[htbp]
    \centering
    \includegraphics[width=\textwidth]{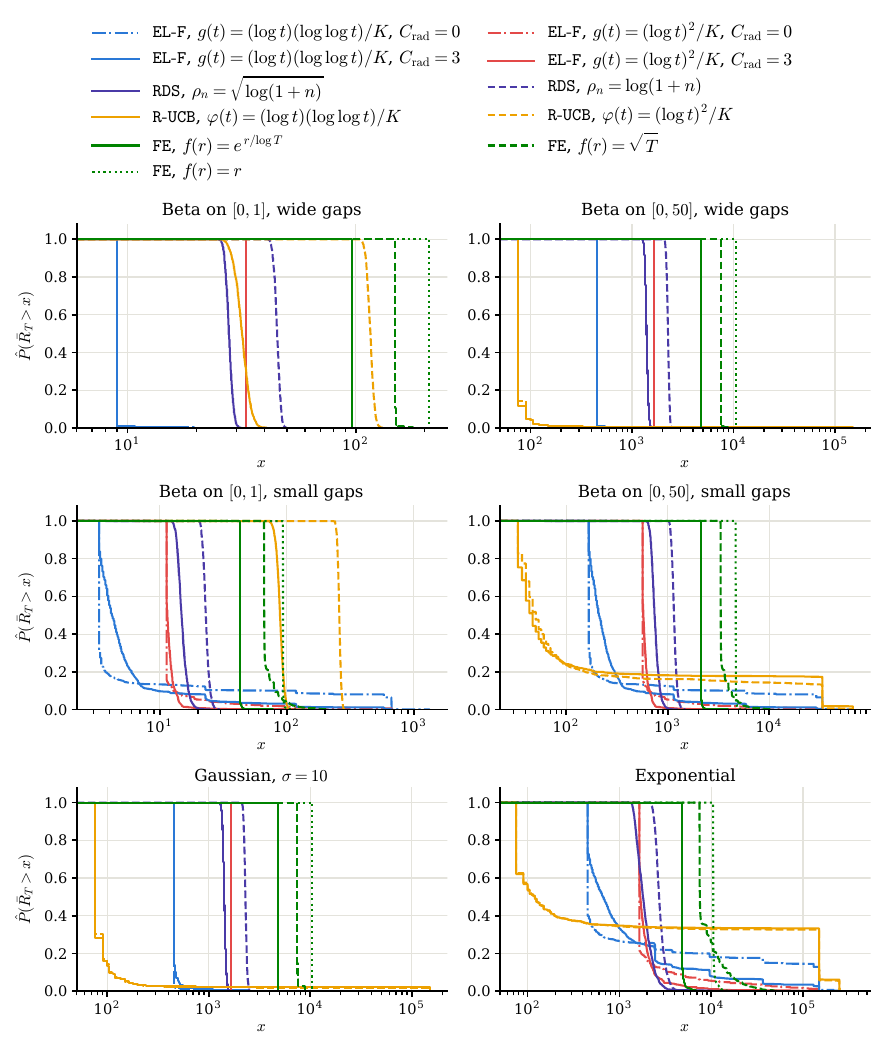}
    \caption{Empirical tail $\hat P(R_T > x)$ of the final pseudo-regret over $1{,}024$ runs on all six instances, for all the algorithms ($x$ on a logarithmic scale).}
    \label{fig:regret_tail_all}
\end{figure}

\begin{figure}[htbp]
    \centering
    \includegraphics[width=\textwidth]{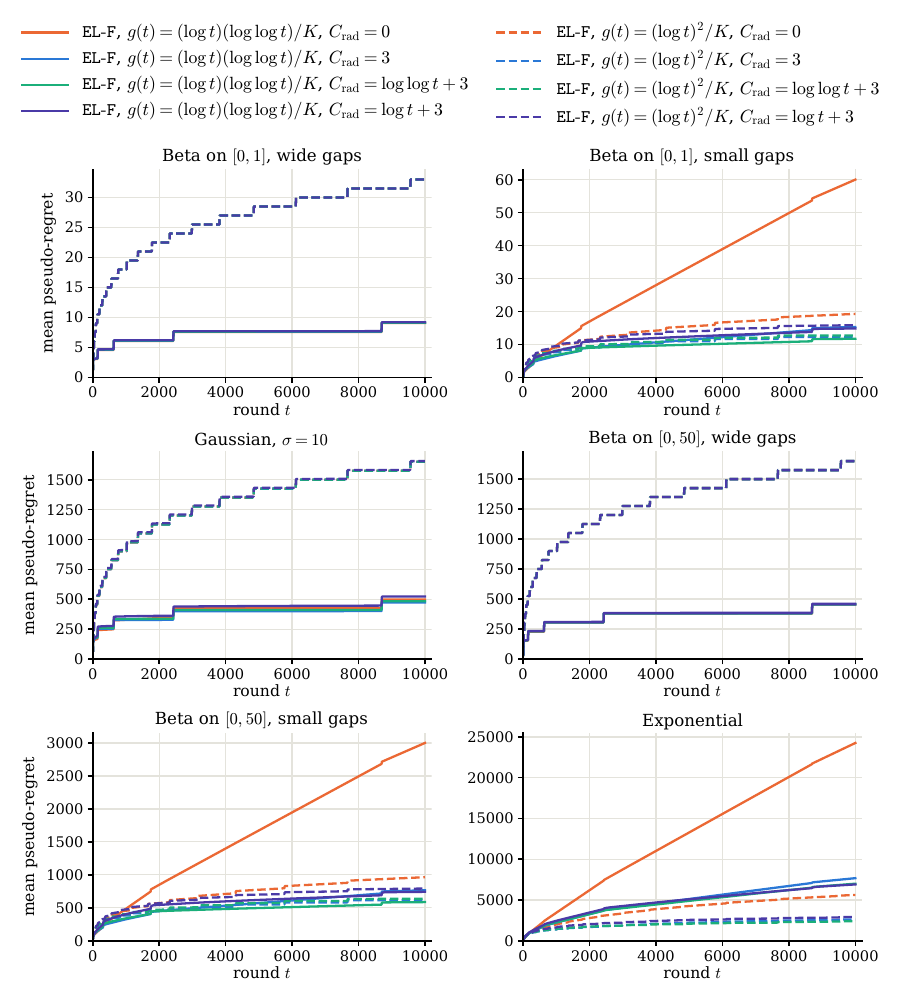}
    \caption{Mean pseudo-regret for the radius ablation of \texttt{EL-F} over $1{,}024$ runs on all six instances, under exploration schedules $g_1$ and $g_2$.}
    \label{fig:crad_ablation_curves}
\end{figure}

\begin{figure}[htbp]
    \centering
    \includegraphics[width=\textwidth]{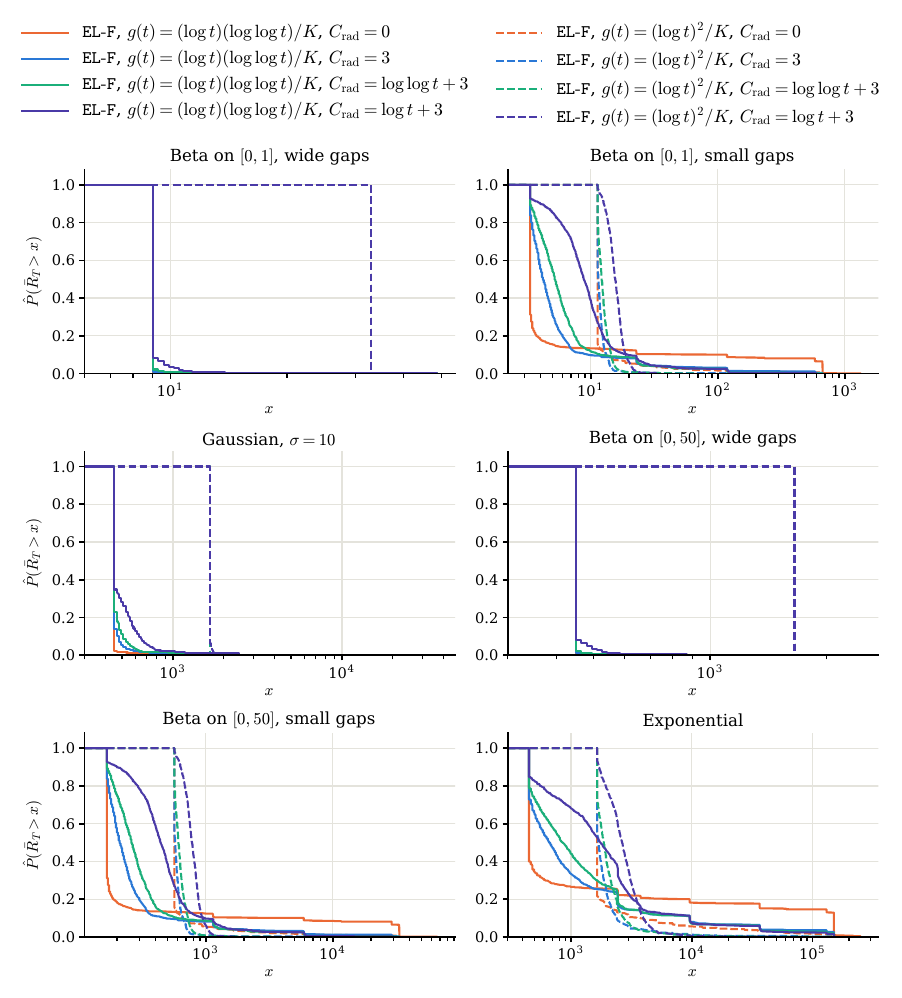}
    \caption{Empirical tails of final pseudo-regret at $T=10{,}000$ for the radius variants in \Cref{fig:crad_ablation_curves}, over $1{,}024$ runs ($x$ on a logarithmic scale).}
    \label{fig:crad_ablation_tail}
\end{figure}

\begin{table}[htbp]
\centering
\caption{Mean and median of the pseudo-regret
$R(T) = \sum_a \Delta_a N_a(T+1)$ at $T = 10{,}000$
over 1{,}024 independent runs, on all six instances,
where $g_1(t) = \frac{1}{K}\log t\,\log\log t$ and
$g_2(t) = \frac{1}{K}(\log t)^2$, with $K$ the
number of arms.}
\label{tab:crad-ablation}
\resizebox{\ifdim\width>\textwidth\textwidth\else\width\fi}{!}{%
\begin{tabular}{@{}l*{12}{r}@{}}
\toprule
Policy
& \multicolumn{4}{c}{Beta (wide gaps)}
& \multicolumn{4}{c}{Beta (small gaps)}
& \multicolumn{2}{c}{Gaussian}
& \multicolumn{2}{c}{Exponential} \\
\cmidrule(lr){2-5}
\cmidrule(lr){6-9}
\cmidrule(lr){10-11}
\cmidrule(lr){12-13}
& \multicolumn{2}{c}{$[0,1]$}
& \multicolumn{2}{c}{$[0,50]$}
& \multicolumn{2}{c}{$[0,1]$}
& \multicolumn{2}{c}{$[0,50]$}
& \multicolumn{2}{c}{}
& \multicolumn{2}{c}{} \\
\cmidrule(lr){2-3}
\cmidrule(lr){4-5}
\cmidrule(lr){6-7}
\cmidrule(lr){8-9}
& mean & med. & mean & med.
& mean & med. & mean & med.
& mean & med. & mean & med. \\
\midrule
\texttt{EL-F}, $g_1$, $C_{\mathrm{rad}} = 0$
& \textbf{9.06} & \textbf{9.00}
& \textbf{453} & \textbf{450}
& 60.0 & \textbf{3.33}
& 3{,}001 & \textbf{167}
& 498 & \textbf{450}
& 24{,}255 & \textbf{450} \\

\texttt{EL-F}, $g_1$, $C_{\mathrm{rad}} = 3$
& 9.10 & \textbf{9.00}
& 455 & \textbf{450}
& 15.3 & 4.20
& 765 & 210
& \textbf{474} & \textbf{450}
& 7{,}672 & 660 \\

\texttt{EL-F}, $g_1$, $C_{\mathrm{rad}} = \log\log t + 3$
& 9.13 & \textbf{9.00}
& 456 & \textbf{450}
& \textbf{11.7} & 5.20
& \textbf{586} & 260
& 484 & \textbf{450}
& 6{,}994 & 855 \\

\texttt{EL-F}, $g_1$, $C_{\mathrm{rad}} = \log t + 3$
& 9.20 & \textbf{9.00}
& 460 & \textbf{450}
& 14.8 & 8.83
& 742 & 442
& 523 & \textbf{450}
& 6{,}906 & 1{,}755 \\

\midrule
\texttt{EL-F}, $g_2$, $C_{\mathrm{rad}} = 0$
& 33.0 & 33.0
& 1{,}650 & 1{,}650
& 19.3 & 11.3
& 964 & 567
& 1{,}650 & 1{,}650
& 5{,}611 & 1{,}650 \\

\texttt{EL-F}, $g_2$, $C_{\mathrm{rad}} = 3$
& 33.0 & 33.0
& 1{,}650 & 1{,}650
& 12.3 & 11.5
& 614 & 573
& 1{,}650 & 1{,}650
& 2{,}607 & 1{,}695 \\

\texttt{EL-F}, $g_2$, $C_{\mathrm{rad}} = \log\log t + 3$
& 33.0 & 33.0
& 1{,}650 & 1{,}650
& 12.7 & 12.2
& 636 & 610
& 1{,}651 & 1{,}650
& \textbf{2{,}389} & 1{,}875 \\

\texttt{EL-F}, $g_2$, $C_{\mathrm{rad}} = \log t + 3$
& 33.0 & 33.0
& 1{,}650 & 1{,}650
& 15.8 & 15.6
& 791 & 780
& 1{,}656 & 1{,}650
& 2{,}911 & 2{,}655 \\

\bottomrule
\end{tabular}%
}
\end{table}

\begin{table}[htbp]
\centering

\caption{Percentage of runs satisfying the trapped-run criterion for \texttt{EL-F} at $T=10{,}000$, over $1{,}024$ independent runs. A run is classified as trapped if the optimal arm is selected fewer than $T/2=5{,}000$ times.}

\label{tab:trapped-all}
\begin{tabular}{@{}l*{6}{r}@{}}
\toprule
Policy & \multicolumn{2}{c}{Beta (wide gaps)} & \multicolumn{2}{c}{Beta (small gaps)} & \multicolumn{1}{c}{Gaussian} & \multicolumn{1}{c}{Exponential} \\
\cmidrule(lr){2-3} \cmidrule(lr){4-5}
 & $[0,1]$ & $[0,50]$ & $[0,1]$ & $[0,50]$ & & \\
\midrule
\texttt{EL-F}, $g_1$, $C_{\mathrm{rad}} = 0$ & 0.0 & 0.0 & 8.1 & 8.1 & 0.0 & 14.6 \\
\texttt{EL-F}, $g_1$, $C_{\mathrm{rad}} = 3$ & 0.0 & 0.0 & 1.2 & 1.2 & 0.0 & 3.5 \\
\texttt{EL-F}, $g_2$, $C_{\mathrm{rad}} = 0$ & 0.0 & 0.0 & 0.9 & 0.9 & 0.0 & 2.0 \\
\texttt{EL-F}, $g_2$, $C_{\mathrm{rad}} = 3$ & 0.0 & 0.0 & 0.0 & 0.0 & 0.0 & 0.4 \\
\bottomrule
\end{tabular}
\end{table}

\end{document}